\documentclass[11pt]{article}
\pdfoutput=1

\usepackage{microtype}
\usepackage{graphicx}
\usepackage{enumerate}
\usepackage{amsthm,amsmath,amsfonts}
\usepackage{float}
\usepackage[caption = false]{subfig}
\usepackage{booktabs} 
\usepackage{bm}
\usepackage{diagbox}
\usepackage{algorithm,algorithmic}
\usepackage{thm-restate}
\usepackage[algo2e,ruled,vlined]{algorithm2e}

\usepackage{geometry}
\usepackage{amssymb}
\usepackage{tikz}
\usepackage{url}
\usepackage{setspace}
\usepackage[pdftex,bookmarksnumbered,bookmarksopen,
colorlinks,citecolor=blue,linkcolor=blue,urlcolor=blue]{hyperref}
\usepackage{framed}
\usepackage{xcolor}
\usepackage{soul}
\usepackage{longtable}

\usepackage{times}
\usepackage{enumitem}
\usepackage{varwidth}
\usepackage{graphicx}
\usepackage{wrapfig}
\usepackage{enumerate}
\usepackage{caption}

\usepackage{amssymb}
\usepackage{graphicx}
\usepackage{url}
\usepackage{setspace}
\usepackage{framed}
\usepackage{xcolor}
\usepackage{soul}
\usepackage{longtable}

\usepackage{times}
\usepackage{enumitem}
\usepackage{varwidth}
\usepackage{graphicx}
\usepackage{wrapfig}
\usepackage{enumerate}

\usepackage[utf8]{inputenc} 
\usepackage[T1]{fontenc}    
\usepackage{url}            
\usepackage{booktabs}       
\usepackage{amsfonts}       
\usepackage{nicefrac}       
\usepackage{microtype}      

\usepackage{graphicx}
\usepackage{appendix}
\usepackage{mdwlist}
\usepackage{xspace}
\usepackage{color}
\usepackage{mathrsfs}

\usepackage{booktabs}
\usepackage{comment}

\usepackage{multirow}

\newcommand{\Appendix}[1]{the full version for}

\newtheorem{theorem}{Theorem}[section]

\newtheorem{remark}{Remark}
\newtheorem{example}{Example}

\newtheorem{definition}{Definition}

\newcommand{\x}{\bm{x}}
\newcommand{\y}{\bm{y}}
\newcommand{\z}{\mathbf{z}}

\newcommand{\I}{\mathbf{I}}
\newcommand{\bk}{\mathbf{k}}
\newcommand{\K}{\mathcal{K}}

\newcommand{\R}{\mathbb{R}}
\newcommand{\bbR}{\mathbb{R}}

\renewcommand{\comment}[1]{}
\newcommand{\red}[1]{}

\newcommand{\cA}{\mathcal{A}}

\newcommand{\cC}{\mathcal{C}}
\newcommand{\cD}{\mathcal{D}}
\newcommand{\cF}{\mathcal{F}}

\newcommand{\cH}{\mathcal{H}}

\newcommand{\cK}{\mathcal{K}}

\newcommand{\cN}{\mathcal{N}}

\newcommand{\cP}{\mathcal{P}}
\newcommand{\cQ}{\mathcal{Q}}

\newcommand{\cX}{\mathcal{X}}

\newcommand{\bbE}{\mathbb{E}}

\definecolor{colorY}{rgb}{0.7 , 0.7 , 0.2}

\DeclareMathOperator*{\argmax}{argmax}
\DeclareMathOperator*{\argmin}{argmin}

\newenvironment{proofoutline}{\noindent{\emph{Proof Sketch. }}}{\hfill$\square$\medskip}

\title{Offline-to-online hyperparameter transfer for stochastic bandits}

\author{{Dravyansh Sharma}$^\dag\!$ \and {Arun Sai Suggala}\footnote{Google LLC.\\\hspace*{4mm}$^\dag\!$ TTIC. Part of the work was done when DS was an intern at Google.}}

\date{}

\begin{document}

\maketitle

\begin{abstract}%
  Classic algorithms for stochastic bandits typically use hyperparameters that govern their critical properties such as the trade-off between exploration and exploitation. Tuning these hyperparameters is a problem of great practical significance. However, this is a challenging problem and in certain cases is information theoretically impossible. To address this challenge, we consider a practically relevant transfer learning setting where one has access to offline data collected from several bandit problems (tasks) coming from an unknown distribution over the tasks. Our aim is to use this offline data to set the hyperparameters for a new task drawn from the unknown distribution. We provide bounds on the inter-task (number of tasks) and intra-task (number of arm pulls for each task) sample complexity for learning near-optimal hyperparameters on unseen tasks drawn from the distribution. Our results apply to several classic algorithms, including tuning the exploration parameters in UCB and LinUCB and the noise parameter in GP-UCB. Our experiments indicate the significance and effectiveness of the transfer of hyperparameters from offline problems in online learning with stochastic bandit feedback.
\end{abstract}

\section{Introduction}
 Bandit optimization is a very important framework for sequential decision making with numerous applications, including recommendation systems~\cite{li2010contextual}, healthcare~\cite{tewari2017ads}, AI for Social Good (AI4SG)~\cite{mate2022field}, hyperparameter tuning in deep learning~\cite{bergstra2011algorithms}. Over the years, numerous works have designed optimal algorithms for bandit optimization in various settings~\cite{garivier2011kl, abbasi2011improved, whitehouse2023improved}. 
 One of the key challenges in deploying these algorithms in practice is setting their hyperparameters appropriately. Some examples of these hyperparameters include the confidence width parameter in UCB~\cite{auer2002finite} and the exploration parameter in $\epsilon$-greedy algorithms. While one could rely on theory-suggested hyperparameters, they often turn out to be too optimistic and lead to suboptimal performance in practice. This is because these hyperparameters are meant to guard the algorithm against worst-case problems and are not necessarily optimal for typical problems arising in a domain. To see this, consider a Multi-Armed Bandit (MAB) instance with two arms, with rewards of each arm drawn from a uniform distribution over an unknown unit width interval. Suppose, the  problems encountered in practice are such that the two distributions are well separated without any overlap. Then, the optimal choice of the UCB confidence width parameter is $0$ (\emph{i.e.,} there is no need for exploration). This is in contrast to the theory suggested choice of $1$, which incurs an exploration penalty. This simplistic example illustrates the importance of choosing hyperparameters appropriately in practice. 


Although hyperparameter selection is well studied in offline learning~\cite{stone1974cross, efron1992bootstrap}, it is still an emerging field in bandit optimization. Several recent studies have attempted to address this problem. One particularly popular approach in this line of works is the design of meta-bandit algorithms that treat each hyperparameter as an expert and adaptively select the best expert by running another bandit algorithm on top (also known as \emph{corralling} algorithms)~\cite{agarwal2017corralling}. However, the current algorithms for corralling and their theoretical guarantees are not satisfactory. For instance, consider the stochastic MAB problem. Let's again consider the problem of picking the UCB confidence width parameter. The corralling algorithm of~\cite{agarwal2017corralling} incurs a regret overhead of $O(\sqrt{MT})$ compared to the regret of the best hyperparameter (where $M$ is the size of the hyperparameter search space). This overhead can be quite large when $M$ is large. Furthermore, their algorithm requires the base algorithms to satisfy certain stability conditions. However, UCB is known to satisfy this condition \emph{only} for certain values of the hyperparameter. Recent works of~\cite{arora2021corralling} improved the regret guarantees of \cite{agarwal2017corralling} by considering stochastic bandits. But their regret bounds, when specialized to UCB confidence width tuning, are worse than the regret bounds obtained using the theory suggested hyperparameter. 

In this work, we first ask the following question: \emph{Is it possible to choose good hyperparameters for a given bandit algorithm that perform nearly as well as the best possible hyperparameters, on any given problem instance?\footnote{A problem instance is defined as a set of arms and their reward distributions.}} Interestingly, we answer this question negatively, showing that even in the simplest problem of stochastic multi-armed bandits (MABs), determining the best hyperparameter for UCB is information-theoretically impossible. 



To address the challenge of hyperparameter selection in bandit algorithms, we propose a data-driven approach. We assume that the learner has access to historical data from \emph{similar} problem instances to the one at hand. This is a reasonable assumption in applications such as hyperparameter tuning in deep learning, where we often have offline data collected from previous hyperparameter tuning runs on similar tasks. Our goal is to leverage this side information to find a good hyperparameter that approximately minimizes the expected reward of the algorithm.

We develop a formal framework for hyperparameter tuning of stochastic bandit algorithms in the presence of side information. In our framework, we assume that problem instances in an application domain are drawn from an unknown distribution, and that the learner has access to historical data collected from problem instances sampled from this distribution. Effective hyperparameter tuning in our framework corresponds to minimizing two quantities of interest while learning near-optimal hyperparameters  --- (a) the inter-task sample complexity, which is the number of tasks or problem instances that the learner needs in the historical data, and (b) the intra-task  sample complexity which corresponds to the number of arm pulls/data points collected from each task. In this work, we derive inter-task and intra-task sample complexity bounds for learning hyperparameters that are provably close in performance to the domain-specific optimal choice, on unseen tasks drawn from the distribution. Our results apply to several classic bandit learning algorithms, including tuning the confidence width or exploration parameters in UCB (multi-armed bandits) and LinUCB~\cite{abbasi2011improved}, and the noise parameter in GP-UCB (GP bandits)~\cite{srinivas2010gaussian}. 

Here is a summary of our contributions:
\begin{itemize}
\itemsep=-2pt
    \item We show that learning near-optimal hyperparameters for every problem instance is impossible without suffering sub-optimal regret, by exhibiting a lower bound in the case of MAB with arm rewards sampled from a Gaussian distribution with unknown variance (Section \ref{sec:lower-bound}).
    \item We present a formal framework for hyperparameter transfer in the bandit setting, and provide general tools for bounding the inter-task and intra-task sample complexity (Sections \ref{sec:framework} and \ref{sec:sample-complexity}).
    \item We instantiate our tools to obtain concrete sample complexity bounds for tuning the exploration parameter in the UCB and LinUCB algorithms (Section \ref{sec:exploration-parameter}). We further obtain sharper sample complexity for the special case of MAB with Bernoulli rewards, and a larger sample complexity for simultaneously learning the prior and the hyperparameter. We also extend our techniques to tune the noise parameter in GP-UCB (Section \ref{sec:gpucb}).
    \item Our experimental results (Section \ref{sec:expts}) demonstrate the effectiveness and significance of transferring hyperparameter knowledge from offline data to online bandit optimization.
\end{itemize}


\section{Related Work}

\paragraph{Transfer Learning.} 
Transfer learning in stochastic bandits, across several related tasks, has received recent attention from the community~\cite{yogatama2014efficient}. But unlike our work, majority of these works have focused on transferring the knowledge of the reward model from one task to another. Our work is complementary to these works as we focus on transferring the knowledge of hyperparameters. \cite{kveton2020meta} considered a similar setting as us, where the bandit instances are sampled from an unknown distribution. The authors studied model learning using gradient descent for simpler policies like explore-then-commit (under nice arm-reward distributions) for which the expected reward as a function of the parameter is concave and easier to optimize. This structure does not hold for several typical reward distributions (e.g.\ Bernoulli). Our results hold for general unknown reward distributions and more powerful UCB-based algorithm families. 
\cite{azar2013sequential} considered a sequential arrival of tasks, but for the much simpler setting of finitely many models or bandit problems (finite $\Pi$ in our notation) which are known to the learner. 
In a similar line of work, \cite{khodak2023meta} designed a meta-algorithm to set the initialization, and other hyperparameters of Online Mirror Descent algorithm, based on past tasks.  \cite{swersky2013multi, wang2022pre} showed that learning priors from historic data helped improve the performance of GP-UCB. We note that these works are mostly empirical in nature and do not provide any sample complexity bounds. Apart from bandit feedback, transfer learning of hyperparameters in similar tasks has also been studied in online learning with {\it full information} feedback \cite{finn2019online,khodak2019adaptive}. 

\paragraph{Corralling.} Corralling bandits \cite{agarwal2017corralling,cutkosky2020upper,arora2021corralling,luo2022corralling} is a recent line of work which involves design of bandit algorithms that compete with the best  algorithm from among a finite collection of bandit learning algorithms. The remarkable techniques have a wide range of applications, including model selection and adapting to misspecification \cite{foster2020adapting,pacchiano2020model}. These results also imply an approach for online hyperparameter tuning given a finite collection of hyperparameters. In contrast, in this work we study bandit hyperparameter selection over continuous parameter domain. Also, we consider a multitask setting, where the goal is to learn a good hyperparameter setting for a collection of tasks. \cite{ding2022syndicated} developed a corralling style meta algorithm for hyperparameter selection in contextual bandits. \cite{angermueller2020population} provided an algorithm that performs corralling to pick the best bandit algorithm for the design of biological sequences. However, both these works are mostly empirical in nature. \cite{kang2024online} consider the continuous parameter domain, but make Lipschitzness assumptions, which are not guaranteed for arbitrary distributions. In fact, even for Bernoulli arm rewards in the MAB setting, the expected rewards can be shown to be a piecewise constant function of the exploration parameter, which is not Lipschitz, and consequently their zooming algorithm based approach cannot be applied.  


\paragraph{Model selection in Bandits.} \cite{foster2019model} study the related problem of model selection in contextual bandits, where the class of models consist of a finite nested sequence of and obtain regret guarantees that scale with the complexity of the smallest class containing the true model. Further work studies design of  algorithms with optimal regret guarantees \cite{pacchiano2020model, marinov2021pareto, pacchiano2022best, krishnamurthy2021optimal}.

\paragraph{Data-driven Hyper-parameter Selection.} Our techniques  are based on and extend the data-driven algorithm configuration paradigm \cite{gupta2017pac,balcan2017learning,balcan2020data}.
 The approach has been found useful application in hyperparameter tuning with formal guarantees across a  number of applications \cite{balcan2018data,balcan2021data,balcan2021much,blum2021learning,balcan2022provably,bartlett2022generalization,balcan2023analysis,balcan2024learning}. We obtain a novel general derandomization based sample complexity bound motivated by tuning algorithms for stochastic bandit problems, but may be of independent interest. Our application to the bandit learning problem requires a distinction between inter-task and intra-task complexity, while previous applications involve simpler notions of sample complexity. While prior work on online data-driven algorithm design typically makes additional smoothness assumptions which roughly correspond to niceness of the online sequences \cite{balcan2018dispersion,sharma2020learning}, we work with arbitrary (non-smooth) distributions and exploit access to offline data from related tasks. Another distinction from prior work is that we study randomized/stochastic problem instances (given by arm reward distributions) requiring new analytical tools, while prior work typically involves deterministic problem instances (once drawn from the problem distribution).

 \paragraph{Hyperparameter free Bandit Algorithms.} An alternate line of work is the {\it hyperparameter free bandit algorithms} such as UCB-V that achieve optimal regret guarantees~\cite{audibert2009exploration,mukherjee2018efficient,zimmert2021tsallis,ito2021parameter}. But these algorithms have certain hidden hyperparameters that need to tuned appropriately. To the best of our knowledge, there are no bandit algorithms that are truly hyperparameter free while achieving optimal regret guarantees. Besides, our work is tackling a fundamentally different problem compared to some of the works on parameter-free bandits. For instance, \cite{zimmert2021tsallis, ito2021parameter} care about worst-case regret bounds. In contrast, our framework is adaptive and tries to find the best hyperparameter for any given bandit instance. This distinction is important and leads to a huge difference in practical performance (as illustrated in our experiments). For the case of online learning in full information, \cite{cutkosky2017stochastic} proposed hyperparameter free algorithms.

 \paragraph{Distribution Shift.} Our transfer-learning result works for arbitrary collections of arm-reward distributions, i.e. different tasks can have very different arm-rewards, but effectively assumes a fixed but unknown distribution over the tasks. In contrast, work on robustness to distribution shifts typically assumes that the possible arm-reward distributions are similar e.g. in terms of KL-divergence \cite{si2020distributionally,husain2024distributionally}, or Wasserstein distance \cite{shen2023wasserstein}, and the results are therefore meaningful only for small amounts of distribution shifts. They do not need the meta-distribution and are more “worst-case” in that sense. Furthermore, this line of work does not study hyperparameter transfer, instead they study regret minimization and guard against worst case distribution shifts assuming small divergence.

\paragraph{Meta-learning Bayesian Bandits.} The classic setting of Bayesian bandits assumes that the prior is ``known'' in advance. Even the recent work on meta-learning Bayesian bandits \cite{wang2018regret,azizi2023meta} makes strong assumptions about the form of the prior e.g.\ multivariate Gaussian distribution. We do not make any assumptions about the nature of the ``prior”, and we furthermore assume it is unknown to the learner. \cite{peleg2022metalearning} have even further assumptions on known bounds on actions and eigenvalues, and on `monotonicity’ of distribution. Even though our regret definition is similar to Bayesian regret, the sample complexity of the hyperparameter transfer problem that we study is very  different from the Bayesian regret minimization problem setting. In our setting, we get to see data from multiple bandit instances, and the goal is to learn a hyperparameter for future instances. 


\section{Preliminaries}

A stochastic online learning problem with bandit feedback consists of a repeated game played over $T$ rounds. In round $t\in[T]$, 
the player  plays an arm $a_t\in[n]$ from a finite set of $n$ arms and the environment simultaneously selects a reward function $r_t:
[n]\rightarrow\bbR_{\ge0}$. In the stochastic setting, the environment draws a reward vector $\mathbf{r}_t$ as an independent sample from some fixed (but unknown) distribution over $\bbR_{\ge0}^n$, and the reward is simply $\mathbf{r}_t(a)=\mathbf{r}_{ta}$ for $a\in[n]$. Finally, the player observes and accumulates the reward $r_t(a_t)$ corresponding (only) to the selected arm $a_t$.  This   is the well-studied stochastic  MAB  setting. A standard measure of the player's performance is the pseudo-regret given by

$$R_T = \max_{a\in[n]} \bbE\left[\sum_{t=1}^T r_t(x_t,a)-r_t(x_t,a_t)\right]$$
\noindent where the expectation is taken over the randomness of both the player and the environment. The expected average regret is given by $\overline{R_T}:=R_T/T$. Let $\mu_i$ denote the mean reward of arm $i\in[n]$ and $\Delta_i:=\max_j\mu_j - \mu_i$ denote the gap in the mean arm reward relative to the best arm.  We will use the shorthand $\mu_{[n]}\{\mu_{l^*}\rightarrow \mu'_{l^*}\}$ to denote that the $l^*$-th entry of the tuple $\mu_{[n]}=(\mu_1,\dots,\mu_n)$ is updated to $\mu'_{l^*}$.




\section{Impossibility of Hyperparameter Tuning}\label{sec:lower-bound}

We now present lower bounds showing that optimal hyperparameter selection is not always possible. Before we do that, we first quantify the notion of ``optimal'' hyperparameter selection. Consider a family of online learning algorithms $\cA = \{A_\rho: \rho\in\cP \}$, parameterized by hyperparameter $\rho$. An example of $\cA$ is the set of UCB policies, with $\rho$ being the scale parameter multiplying the confidence width. Let $\Pi$ be a collection of stochastic multi-armed bandit problems. In hyperparameter tuning, our goal is to design a meta algorithm ($\widetilde{A}$) for choosing an appropriate $\rho$ which can compete with the best possible hyperparameter, on any problem instance in $\Pi$. To be precise, we want $\widetilde{A}$ to satisfy the following consistency condition for any $P \in \Pi$: $\lim_{T\to\infty} R_T(\widetilde{A}; P)/R_T(A_{\rho^*}; P)=1$.  Here, $\rho^* = \argmin_{\rho\in \cP} R_T(A_{\rho}; P)$ is the best hyperparameter for problem $P$. We note that similar notions of consistency have been studied in the context of hyperparameter selection in offline learning~\cite{kearns1995bound, yang2007consistency}. In fact, hyperparameter selection techniques such as cross validation are known to satisfy such consistency properties.

The following result shows that consistency is not possible even in the simplest problem of MAB with rewards sampled from a Gaussian distribution.
\begin{restatable}{theorem}{thmlb}
    \label{thm:impossibility_hp_tuning}
    Let $\Pi$ be the set of MAB problems with arm rewards sampled from Gaussian distributions with variance belonging to the set $[0, B^2].$ Let $\cA$ be the set of UCB policies, with $\rho$ being the scale parameter multiplying the confidence width. Then for any meta algorithm $\widetilde{A}$, there exists a problem $P\in \Pi$ which satisfies the following bound:
    $\lim_{T\to\infty} R_T(\widetilde{A}; P)/R_T(A_{\rho^*}; P) > 1.$
    \end{restatable}
\noindent The proof builds on standard distribution dependent lower bounds~\cite{cappe2013kullback}, and is located in the appendix. This result motivates our framework which uses offline bandit runs on similar problem instances for hyperparameter transfer. 

\section{Formal framework for transfer learning}\label{sec:framework}

Given a stochastic online learning problem (say multi-armed bandits), let $\Pi$ denote the set of problems (or tasks) of interest. That is, each $P\in \Pi$ defines an online learning problem. For example, if $\Pi$ is a collection of stochastic multi-armed bandit problems, then $P$ could correspond to a fixed product distribution of arm rewards.
We also fix a (potentially infinite) family of online learning algorithms $\cA$, parameterized by a set $\cP\subseteq\R^d$ of $d$  real (hyper-)parameters. Let $A_\rho$ denote the algorithm in the family $\cA$ parameterized by $\rho\in\cP$. For any fixed time horizon $T$, the performance of any fixed algorithm on any fixed problem  is given by some bounded loss metric (e.g.\ the expected average regret of the algorithm) $l_T:\Pi\times\cP\rightarrow[0,H]$\red{is it okay to assume expected  regret is always non-negative?}, i.e.\ $l_T(P,\rho)$ measures the performance on problem  $P\in\Pi$ of algorithm $A_\rho\in\cA$. The utility of a fixed algorithm $A_\rho$ from the family is given by $l^\rho_T:\Pi\rightarrow[0,H]$, with $l^\rho_T(P)=l_T(P,\rho)$. We will be interested in the structure of the {\it dual class} of functions $l^P_T:\cP\rightarrow[0,H]$, with $l^P_T(\rho)=l^\rho_T(P)$, which measure the performance of all algorithms of the family for a fixed problem  $P\in\Pi$. 

We assume an unknown distribution $\cD$ over the problems in $\Pi$. We further have a collection of ``offline'' problems which we can  use to learn a good value of the algorithm parameter $\rho$ that works well on average on a random ``test'' problem drawn from $\cD$. We are interested in the sample complexity of the number of ``offline'' problems that are sufficient to learn a near-optimal $\rho$ over $\cD$. Formally, the learner is given a collection $\{P_1,\dots,P_N\}\sim\cD^N$ of {\it offline} problems for each of which the rewards may be collected according to some policy (not necessarily the same policy as in the test problem) in an online game with time horizon $T_o$, the {\it intra-task}  complexity. The learner learns a hyperparameter $\hat{\rho}$ based on these offline runs.
 A {\it test} problem is given by a random $P\sim \cD$ on which the loss metric $l_T(P,\hat{\rho})$ is measured over an online game of $T$ rounds.  The $(\epsilon,\delta)$ {\it sample complexity} of the learner is the number $N$ of offline problems sufficient to guarantee that learned parameter $\hat{\rho}$ is near-optimal with high probability, i.e.\ with probability at least $1-\delta$,

 $$\Big\lvert\bbE_{P\sim \cD}[l_T(P,\hat{\rho})]-\min_{\rho\in\cP}\bbE_{P\sim \cD}[l_T(P,{\rho})]\Big\rvert\le \epsilon.$$

\subsection{Derandomized dual complexity}

We will define a useful quantity to measure the inherent challenge in learning the best hyperparameter for an unknown problem distribution $\cD$. For any offline problem $P$, let $\mathbf{z}$ denote the random coins used in drawing the contexts and arm rewards according to the corresponding distribution $D_{P}$, and $l^{P,\mathbf{z}}_T(\rho)$ denote the corresponding derandomized dual function, i.e.\ $l^{P}_T(\rho)=\bbE_{\mathbf{z}\sim D_{P}}[l^{P,\mathbf{z}}_T(\rho)]$. Intuitively, we can think of fixing $\mathbf{z}$ as drawing the rewards according to $D_{P}$ for the entire time horizon $T$ in advance (revealed as usual to the online learner) and taking an expectation over $\mathbf{z}$ gives the expected loss or reward according to $P$. More concretely, we have $\mathbf{z}=(z_{ti})_{i\in[n],t\in[T]}$ with each $z_{ti}$ drawn i.i.d. from the uniform distribution over $U([0,1])$. If $D_i$ denotes the reward distribution for arm $i$ and $F_i$ its cumulative density function, then the reward $\mathbf{r}_{ti}$ is given by $F_i^{-1}(z_{ti})$.

For typical parameterized stochastic bandit algorithms, we will show that $l^{P,\mathbf{z}}_T(\rho)$ is a piecewise constant function of $\rho$, i.e.\ the parameter space $\cP$ can be partitioned into finitely many connected regions $\{\cP_i\}$ such that $l^{P,\mathbf{z}}_T(\rho)=\sum_ic_i\I[\rho\in\cP_i]$ where $c_i\in\R$, $\I[\cdot]$ is the 0-1 valued indicator function, $\bigcup_i \cP_i=\cP$ and $\cP_i \cap \cP_j=\emptyset$ for $i\ne j$. Let $q(f)$ denote the number of pieces $\{\cP_i\}$ over which a piecewise constant function $f:\cP\rightarrow\R$ is defined. We define the {\it derandomized dual complexity} of problem distribution $\cD$ w.r.t.\ algorithms parameterized by $\cP$ as follows.

\begin{definition}\label{def:ddc}
    Suppose the derandomized dual function $l^{P,\mathbf{z}}_T(\rho)$ is a piecewise constant function. The derandomized dual complexity of $\cD$ w.r.t.\ $\cP$ is given by $Q_\cD=\bbE_{P\sim \cD}\bbE_{\z\sim D_P}q(l^{P,\mathbf{z}}_T(\cdot))$.
\end{definition}

\noindent $Q_\cD$ provides a distribution-dependent complexity measure that will be useful to bound the sample complexity as well as the intratask complexity of learning the best parameter over $\cD$. Moreover, it may be empirically estimated over a collection of offline problems sampled from $\cD$.

\section{Sample complexity of bandit hyperparameter tuning}\label{sec:sample-complexity}

We proceed to provide a general sample complexity bound for learning one dimensional parameters, i.e.\ $\cP\subset \R$. We state below our result as a uniform convergence guarantee and provide a proof sketch (full proof is located in the appendix).

\begin{restatable}{theorem}{thmintertask}\label{thm:sample-complexity-inter-task} Consider the above setup for any arbitrary $\cD$ and suppose the derandomized dual function $l^{P,\mathbf{z}}_T(\rho)$ is a piecewise constant function.. For any $\epsilon,\delta>0$, $N$ problems $\{P_i\}_{i=1}^N$ sampled from $\cD$ with corresponding random coins $\{\z_i\}_{i=1}^N$ such that $N=O\left(\left(\frac{H}{\epsilon}\right)^2\left(\log Q_\cD +\log \frac{1}{\delta}\right)\right)$ are sufficient to ensure that with probability at least $1-\delta$, for all $\rho\in\cP$, we have that $$\left\vert \frac{1}{N}\sum_{i=1}^N l_T^{P_i, \mathbf{z}_i}(\rho)-\bbE_{P\sim\cD}l_T^{P}(\rho)\right\vert<\epsilon.$$ 
\end{restatable}

\begin{proofoutline}
Fix a problem  $P\in\Pi$. Fix the random coins $\mathbf{z}$ used to draw the arm rewards according to $D_P$ for the $T$ rounds. 
We will use the piecewise loss structure to bound the  Rademacher complexity, which would imply uniform convergence guarantees by applying standard learning-theoretic results. Let $\rho, \dots, \rho_m$ denote a collection of parameter values, with one parameter from each of the $m\le \sum_{i=1}^Nq(l_T^{P_i, \mathbf{z}_i}(\cdot))$ pieces of the dual class functions $l^{P_i,\mathbf{z}_i}_T(\cdot)$ for $i\in[N]$, i.e.\ across problems in the sample $\{P_1,\dots, P_N\}$ for some fixed randomizations. Let $\cF=\{f_{\rho}:(P,\mathbf{z})\mapsto l^{P,\mathbf{z}}_T(\rho)\mid \rho\in\cP\}$ be a family of functions on a given sample of instances $S=\{P_i,\mathbf{z}_i\}_{i=1}^N$. 

Since the function $f_\rho$ is constant on each of the $m$ pieces, we have the empirical Rademacher complexity,
 $\hat{R}(\cF,S)=\frac{1}{N}\bbE_\sigma\left[\sup_{j\in[m]}\sum_{i=1}^N\sigma_i v_{ij}\right]$,
where $\sigma=(\sigma_1,\dots,\sigma_m)$ is a tuple of i.i.d.\ Rademacher random variables, and $v_{ij}:=f_{\rho_j}(P_i,\mathbf{z}_i)$. Note that $v^{(j)}:=(v_{1j},\dots, v_{Nj})\in[0,H]^N$, and therefore $||v^{(j)}||_2\le H\sqrt{N}$, for all $j\in[m]$. An application of Massart's lemma \cite{massart2000some} now implies $
    \hat{R}(\cF,S)\le H\sqrt{\frac{2\log\sum_{i=1}^Nq(l_T^{P_i, \mathbf{z}_i}(\cdot))}{N}}$.

\noindent Taking average over $S$ and applying the Jensen's inequality, gives a bound on the Rademacher complexity

\begin{align*}
    {R}(\cF,\cD)\le H\sqrt{\frac{2\log N + 2\log Q_\cD}{N}}.
\end{align*}

\noindent Standard Rademacher complexity bounds \cite{bartlett2002rademacher} now imply the desired sample complexity bound.
\end{proofoutline}

\noindent The above result shows that \emph{consistent} hyper-parameter selection is possible as $N\to\infty$, and $\log Q_\cD$ scales sublinearly in $N$. In subsequent sections, we derive explicit bounds on the derandomized dual complexity $Q_{\cD}$ for several key problems of interest. Theorem \ref{thm:sample-complexity-inter-task} not only gives us the inter-task sample complexity, but also reveals an algorithm for finding an $\epsilon$-optimal hyperparameter. In particular, it shows that minimizing $\hat{\rho}:=\argmin_{\rho\in\cP}\sum_{i=1}^{N} l_T^{P_i, \mathbf{z}_i}(\rho)$ is enough to guarantee learning a near-optimal parameter. If all arm rewards are observed in the offline data at every time step (called the ``full information'' setting in online learning), then the intra-task complexity of $T_o=T$ is sufficient to compute $\hat{\rho}$. This is achieved by first estimating $l_T^{P_i, \mathbf{z}_i}(\rho)$ for any $\rho$---by simulating the bandit algorithm using the observed rewards---and then minimizing the above objective. However, a more realistic assumption is that the offline data was gathered under a bandit learning setting, where only the reward for the pulled arm is observed. In this case, as shown in the following Theorem, having sufficiently long intra-task time horizons $T_o$ for the offline tasks allows us to estimate $l_T^{P_i, \mathbf{z}_i}(\rho)$ for any $\rho$.

\begin{restatable}{theorem}{thmintratask}
    \label{thm:abundant_offline_data}There exists an offline policy with $\bbE[T_o]=\min\{n,Q_\cD\}T$ and a hyperparameter tuning algorithm that outputs $\hat{\rho}$ which satisfies the following bound with probability at least $1-\delta$
    \[
    \bbE_{P\sim\cD}[l_T^{P}(\hat\rho)] - \min_{\rho} \bbE_{P\sim\cD}[l_T^{P}(\rho)] \!\leq \!O\!\left(\!\sqrt{\frac{\log{Q_\cD} \!+\! \log{\frac{N}{\delta}} }{N}}\right)\!.
    \]
\end{restatable}

\noindent We now present an offline data collection policy that achieves the above bounds. If the number of arms $n\le Q_\cD$, the offline policy is simply to collect the reward for each of the $n$ arms $T$ times. While this is a reasonable policy when $n$ is small, it is not practical when $n$ is large. However, as we show in the sequel, $Q_{\cD}$ turns out to be quite small, even when $n$ is large. In this case when $Q_\cD<n$, the offline policy operates on the pair $(P,\z)$ by sequentially running the algorithm $A_\rho$ for a single $\rho$ value within each interval where the loss function $l_T^{P,\z}(\cdot)$ is constant.  The algorithm is restarted with a new $\rho$ value after every $T$ rounds. The hyperparameter tuning algorithm simply minimizes $\min_{\rho}\sum_{i=1}^N l_T^{P_i, \mathbf{z}_i}(\rho)$. In the sequel, we assume the offline data is collected from these policies. We defer the investigation of finding the optimal offline data collection policy to future work.


\section{Tuning the exploration parameter}\label{sec:exploration-parameter}

UCB (upper confidence bound) is a well-known algorithm for the  multi-armed bandits problem, which employs simple confidence bands that contain (with high probability) the true mean of the reward distribution for each arm (using Hoeffding's inequality). The key insight is that one can simply pick the arm with the best upper bound on the reward using the confidence band, often aphorized as {\it optimism in the face of uncertainty}. The algorithm is stated in Algorithm \ref{alg:ucb}.\looseness-1

\begin{algorithm}[h] 
    \caption{\textsc{UCB($\alpha$)}\label{alg:ucb}}
    \textbf{Input}: Arms $\{1,\dots,n\}$, max steps $T$\\
    \textbf{Output}: Arm pulls $\{A_t\in[n]\}_{t\in[T]}$\\
    \For{$t=1,\dots,n$}{
        $A_t\gets t$ \tcc*{Pull all arms once.} 
        $\hat{\mu}_{A_t}\gets$ observed reward, $t_{A_t}\gets 1$
    }

    \For{$t=n+1,\dots,T$}{
        $A_t\gets \argmax_{i}\hat{\mu}_i+\sqrt{\frac{\alpha\log t}{t_i}}$\tcc*{$\hat{\mu}_i$ and $t_i$ are the average reward and the number of pulls respectively for arm $i$}
        Update $\hat{\mu}_{A_t}$ and $t_{A_t}$
    }

\end{algorithm}

The exploration-exploitation trade-off of the above algorithm can be tuned using the exploration parameter $\alpha$. Intuitively larger values of $\alpha$ lead to larger confidence bands and therefore more exploration. Classic analysis of the problem give bound on data-independent worst-case regret bounds for $\alpha\ge 1$, and the bounds are tightest for $\alpha=1$.
  However, the best value of $\alpha$ (for a fixed time horizon $T$) is data-dependent and could even be less than 1. This is indeed observed in our empirical validation (see Figure \ref{fig:two-arm-alpha}), where we run UCB for 2-arm bandits ($n=2$) and observe that value of $\alpha$ that optimizes the exploration-exploitation trade-off depends on the data-distribution. 


     

\red{Potential next steps: Similar visualizations of parameter dependence for *-UCB algorithms of interest.}

This motivates the following problem.
Can we identify a `good' value of $\alpha$ for a set of related problems, where each problem corresponds to a fixed arm reward distribution, and the problem itself is sampled from some problem distribution?
One way to formalize this in our context is to consider a {\it training} or {\it offline} phase (say beta version of a news recommendation system \cite{li2010contextual}, or a training phase of architecture search for networks trained on available datasets) where we can learn the exploration parameter and a {\it test} or {\it online} phase (real deployment of the news recommendation system, or held-out  datasets for architecture search) where we employ the learned parameter. A reasonable goal then would be to make the training phase efficient (amount of data needed) and the test phase effective (good learned parameter). Besides the exploration parameter, one could also learn an initialization of the arm means from the offline data, although learning the exploration parameter is still important as the following example illustrates.

\begin{figure}[t]
\centering
 \includegraphics[scale=0.5]{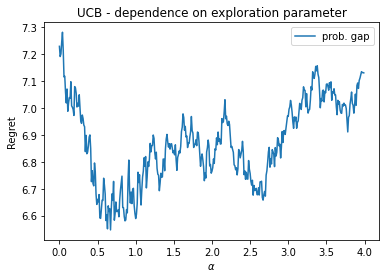}
     
    \caption{Variation of (estimated) expected regret with the exploration parameter $\alpha\in[0,4]$ for two-arm stochastic bandits for symmetric Gaussian distributions.}
    \label{fig:two-arm-mixture}
\end{figure}

\begin{example}
Consider a two-armed bandit problem with two distributions in $\Pi$: $\cD_1$ which places rewards $\cN(\mu_1,\sigma)$ on arm 1 and $\cN(\mu_2,\sigma)$ on arm 2, and $\cD_2$ which places rewards $\cN(\mu_2,\sigma)$ on arm 1 and $\cN(\mu_1,\sigma)$ on arm 2. If we may encounter $\cD_1$ or $\cD_2$ with equal probability, by symmetry no arm is better than the other. Moreover, for $\mu_2=-\mu_1$, we do not learn useful information for initializing mean estimates for the arms. Also, by symmetry, the same exploration parameter is optimal for both $\cD_1$ and $\cD_2$, and therefore for the problem distribution as well. Learning the exploration parameter is  useful  (Figure \ref{fig:two-arm-mixture}).
\end{example}


\noindent We have the following general bound on the derandomized dual complexity $Q_\cD$ for arbitrary distributions $\cD$ over a collection of $n$-arm bandit problems.

\begin{restatable}{theorem}{thmucb}\label{thm:ucb}
    Let $\Pi$ be a collection of multi-armed bandit problems with $n$ arms and $\cD$ be an arbitrary distribution over $\Pi$. Then, for Algorithm \ref{alg:ucb} parameterized by $\alpha$, we have $\log Q_\cD = O(n\log T)$.
\end{restatable}

\noindent A proof of the above result appears in the appendix. We remark that our above result makes use of the fact that the arm rewards in the stochastic bandit problem are iid, which implies that the derandomized dual function has polynomially many discontinuities in $T$ (for fixed $n$). Without this assumption, the number of discontinuities can be as large as $2^{\Omega(T)}$ even for the 2-arm case. 

\begin{algorithm}[t] 
    \caption{\textsc{TunedUCB}($\alpha_{\min},\alpha_{\max}$)\label{alg:tuned-ucb}}
    \textbf{Input}: Parameter interval $[\alpha_{\min},\alpha_{\max}]$, Arm rewards $r_{ijk}, i\in[n], j\in [T]$, $k\in[N]$ from offline data\\
    \textbf{Output}: Learned parameter $\hat{\alpha}$\\
    \For{each problem instance $k\in[N]$}{
    \For{each arm $i\in[n]$}{ 
    $t_i\gets 1$\\
    $R_{i}[-1]\gets (r_{i2k},\dots,r_{iTk})$}
    $A_k\gets$ \textsc{$\alpha$-CriticalPoints}($\alpha_{\min},\alpha_{\max}, (t_1,\dots,t_n), (r_{11k},\dots,r_{n1k}), (R_{1}[-1],\dots,R_{n}[-1])$)\tcc*{Alg \ref{alg:ucb-tuner}, computes critical points for $l_T^{P_k, \mathbf{z}_k}(\alpha)$ on fixed instance $k$.}
    }
    \textbf{return} $\argmin_{\alpha\in\{\alpha_{\min},\alpha_{\max}\}\cup A_1 \cup  \dots \cup  A_N} \sum_{k=1}^Nl^{P_k,\mathbf{z}_k}_T(\alpha)$
    
\end{algorithm}

By Theorems \ref{thm:sample-complexity-inter-task} and \ref{thm:abundant_offline_data}, the above result implies an inter-task sample complexity of $\Tilde{O}(\frac{n\log T}{\epsilon^2})$ and intra-task complexity of $O(nT)$ to learn $\epsilon$-optimal parameter $\alpha$. The sample complexity bound can be achieved by  the ERM algorithm (Algorithm \ref{alg:tuned-ucb}). The key challenge in implementing the ERM is that it involves a minimization over infinitely many $\alpha$, which we resolve by computing the critical points where the sequence of selected arms can change using a recursive procedure. 

\begin{algorithm}[h] 
    \caption{\textsc{$\alpha$-CriticalPoints}($\alpha_{l},\alpha_{h}, t_{[n]}, \mu_{[n]}, R_{[n]}$)\label{alg:ucb-tuner}}
    \textbf{Input}: Parameter interval $[\alpha_{\min},\alpha_{\max}]$, Arm pulls so far $t_i$, Mean rewards so far $\mu_i$, Future arm rewards $R_{i}, i\in[n]$.\\
    \textbf{Output}: Learned parameter $\hat{\alpha}$.\\
    \If{\textsc{length}$(R_i)=0$ for some $i\in[n]$}{\textbf{return}  $\emptyset$}
    $l^*\gets \argmax_{i\in[n]} {\mu}_i+\sqrt{\frac{\alpha_l\log \sum_{j=1}^nt_j}{t_i}}$\\
    $\alpha_{\textsc{next}}\gets \min_{i\in[n],i\ne l^*}\frac{1}{\sum_{j=1}^nt_j}\left(\frac{\mu_{l^*}-\mu_i}{\frac{1}{\sqrt{t_i}}-\frac{1}{\sqrt{t_{l^*}}}}\right)^2$\\
    $\mu'_{l^*}\gets \frac{\mu_{l^*}t_{l^*}+R_{l^*}[0]}{t_{l^*}+1}$\\
    $\alpha^*\gets\min\{\alpha_h, \alpha_{\textsc{next}}\}$\\
        $A_1\gets \alpha\textsc{-CriticalPoints}(\alpha_{l},\alpha^*, t_{[n]}\{t_{l^*}\rightarrow t_{l^*}+1\}, \mu_{[n]}\{\mu_{l^*}\rightarrow \mu'_{l^*}\}, 
         R_{[n]}\{R_{l^*}\rightarrow R_{l^*}[-1]\})$\\
    \If{$\alpha_{\textsc{next}}\ge \alpha_{h}$}{
    
      \textbf{return }  $A_1$ 
    }{
        $A_2\gets \textsc{$\alpha$-CriticalPoints}(\alpha_{\textsc{next}},\alpha_{h}, t_{[n]}, \mu_{[n]}, 
        R_{[n]})$\\
    \textbf{return } $A_1\cup\{\alpha_{\textsc{next}}\}\cup A_2$
    }
\end{algorithm}

\textbf{Algorithm \ref{alg:ucb-tuner}.} An important challenge in implementing ERM is that it involves a minimization over infinitely many $\alpha$'s. We address this by computing the critical points (i.e., points of discontinuities) of the piecewise constant function $l_T^{P_k, \mathbf{z}_k}(\alpha)$. These critical points occur when a slight change to $\alpha$ changes the choice of the arm in UCB (at any time step).  Algorithm~\ref{alg:ucb-tuner} provides an efficient way (runtime proportional to  actual number of discontinuities, i.e. expected time complexity is $O(\cQ_\cD)$) to calculate these critical points, making the ERM approach practical. The key idea is to recursively compute the critical points in any interval $[\alpha_l, \alpha_h]$, by locating the first point $\alpha_{\textsc{next}}$ at which an arm different from the best arm for the left end point $\alpha_l$ is selected.

We can further sharpen the above bound in the case of Bernoulli (more generally categorially) distributed arm rewards. A proof of the following theorem is located in the appendix.

\begin{restatable}{theorem}{thmucbcategorial}\label{thm:ucb-categorical}
    Let $\Pi$ be a collection of multi-armed bandit problems with $n$ arms and $\cD$ be an arbitrary distribution over $\Pi$ such that the arm rewards take categorical values in $\{0,1,\dots,K-1\}$. Then $\log Q_\cD = O(\log KT)$ for \textsc{UCB($\alpha$)}.
\end{restatable}

\noindent {\bf LinUCB.} A similar analysis can be carried out for the LinUCB algorithm for the stochastic contextual bandits problem (Algorithm \ref{alg:linucb}). In this setting, the learner observes a context vector $(x_{t,i})_{i\in[n]}$ in each round $t$, assumed to be drawn from a fixed, unknown distribution, and the reward distribution $D_{t,i}$ for the arm $i$ depends on the context $x_{t,i}$. We obtain the following bound on the inter-task sample complexity.

\begin{restatable}{theorem}{thmlinucb}\label{thm:linucb}
    Let $\Pi$ be a collection of contextual bandit problems with $n$ arms and $\cD$ be an arbitrary distribution over $\Pi$. Then, for \textsc{LinUCB($\alpha$)} (i.e., Algorithm \ref{alg:linucb}), we have $\log Q_\cD = O(T\log n)$.
\end{restatable}

\begin{algorithm}[b] 
    \caption{\textsc{LinUCB($\alpha$)}\label{alg:linucb}}
    \textbf{Input}: Arms $\{1,\dots,n\}$, max steps $T$, feature dimension $d$\\
    \textbf{Output}: Arm pulls $\{A_t\in[n]\}_{t\in[T]}$\\
    $K\gets I_d$, $b\gets 0_d$

    \For{$t=1,\dots,T$}{
        $\theta_t\gets K^{-1}b$\\
        Observe features $x_{t,1},\dots,x_{t,n}\in\R^d$\\
        $A_t\gets \argmax_{i}\theta_t^Tx_{t,i}+\alpha\sqrt{x_{t,i}^TK^{-1}x_{t,i}}$\tcc*{upper confidence bound}
        Observe payoff $p_{t,A_t}$\\
        Update $K\gets K+x_{t,A_t}^Tx_{t,A_t}$ and $b\gets b+x_{t,A_t}^Tp_{t,A_t}$
    }

\end{algorithm}

\noindent The above results apply even when the set of arms changes across different problems in $\Pi$, with $n$ being an upper bound on the number of arms in any problem in $\Pi$. Suppose now that we have a common set of $n$ arms across all the problems in $\Pi$.
In this case, in addition to learning the exploration parameter $\alpha$, we can also learn how to initialize the arm means in the following variant of UCB that incorporates arm priors as  hyperparameters.

Suppose $\alpha\in[\alpha_{\min},\alpha_{\max}]$ be the exploration parameter and $\mathbf{\hat{\mu}^0}=\{\hat{\mu}^0_1,\dots,\hat{\mu}^0_n\}\in\R_{\ge0}^n$ denote prior (initial) arm means. For  any problem  $P\in\Pi$ and randomization $\mathbf{z}$, define the dual class functions $l^{P}_T(\alpha,\mathbf{\hat{\mu}^0})$ and $l^{P,\mathbf{z}}_T(\alpha,\mathbf{\hat{\mu}^0})$ as above. We have the following sample complexity result for learning the hyperparameters $\alpha,\mathbf{\hat{\mu}^0}$ simultaneously.

\begin{restatable}{theorem}{thmucbprior}\label{thm:sample-complexity-with-prior} Consider the above setup for any arbitrary $\cD$. For any $\epsilon,\delta>0$, $N$ problems $\{P_i,\mathbf{z}_i\}_{i=1}^N$ sampled from $\cD$ with $N=O\left(\left(\frac{H}{\epsilon}\right)^2\left((n+T)T\log n +\log \frac{1}{\delta}\right)\right)$ are sufficient to ensure that with probability at least $1-\delta$, for all $\alpha\in[\alpha_{\min},\alpha_{\max}]$, we have that $$\left\vert \frac{1}{N}\sum_{i=1}^N l_T^{P_i, \mathbf{z}_i}(\alpha,\mathbf{\hat{\mu}^0})-\bbE_{P\sim\cD}l_T^{P}(\alpha,\mathbf{\hat{\mu}^0})\right\vert<\epsilon.$$ 
\end{restatable}

\noindent A full proof is located in the appendix, and employs techniques due to \cite{bartlett2022generalization}.

\section{Tuning the noise parameter in GP-UCB}\label{sec:gpucb}

Many problems in reinforcement learning --- for example choosing what ads to display to maximize profit in a click-through model \cite{pandey2006handling}, determining the optimal control strategies for a robot \cite{lizotte2007automatic}, hyperparameter tuning of large machine learning models \cite{bergstra2011algorithms}  --- can be formulated as optimizing an unknown noisy function $f$ that is expensive to evaluate. Seminal work of \cite{srinivas2010gaussian} proposed a simple and intuitive Bayesian approach for this problem called the Gaussian Process Upper Confidence Bound (GP-UCB) algorithm and, under implicit smoothness assumptions on $f$, showed that their algorithm achieves no-regret when optimizing $f$ by a sequence of online evalutations. Their setup formally generalizes the bandit linear optimization problem. 
A crucial parameter of this algorithm is the noise variance parameter $\sigma^2$ which is typically manually tuned. But for many of the above applications one typically ends up repeatedly solving multiple related problem instances. Learning a value of $\sigma^2$ that works well across multiple problem instances is of great practical interest.


\subsection{Setup}
Consider the problem of maximizing a real-valued function $f:\cC\rightarrow \R$ over domain $\cC$ online. In each round $t=1,\dots,T$, the learner selects a point $x_t\in\cC$. The learner wants to maximize $\sum_{t=1}^Tf(x_t)$, and its performance is measured relative to the best fixed point $x^*=\argmax_{x\in\cC} f(x)$. The {\it instantaneous regret} of the learner is defined as $r_t=f(x^*)-f(x_t)$ and the cumulative regret as $R_T=\sum_{t=1}^Tr_t$.


\begin{algorithm}[h] 
    \caption{\textsc{GP-UCB($\sigma^2$)}\label{alg:gp-ucb} \cite{srinivas2010gaussian}}
    \textbf{Input}: Input space $\cC$, GP prior $\mu_0=0,\sigma_0,$ kernel $k(\cdot,\cdot)$ such that $k(\x,\x')\le 1$ for any $\x,\x'\in\cC$, $\{\beta_t\}_{t\in[T]}$.\\
    \textbf{Output}: Point $\{\x_t\in\cC\}_{t\in[T]}$\\

    \For{$t=1,\dots,T$}{
        $\x_t\gets \argmax_{\x\in\cC}{\mu}_{t-1}(\x)+\sqrt{\beta_t}\sigma_{t-1}(\x)$\\
        Observe $y_t=f(\x_t)+\epsilon_t$, $\epsilon_t\sim N(0,\sigma^2)$\\
        $\bk_t(\x)=[k(\x_1,\x) \dots k(\x_t,\x)]^T\\ \K_t=[k(\x_i,\x_j)]_{i,j\in[t]}$\\
        Update $\mu_t(\x)=\bk_t(\x)^T(\K_t+\sigma^2 I)^{-1}\y_t$, where $\y_t=[y_1 \dots y_t]^T$\\
        Update $\sigma_t(\x)=k(\x,\x)-\bk_t(\x)^T(\K_t+\sigma^2 I)^{-1}\bk_t(\x)$
    }

\end{algorithm}

The parameter $\sigma^2$ here is the noise variance, which is unknown to the learner but needed by the algorithm. An upper bound on the true noise $\sigma^2$ is sufficient for the algorithm to work, but a loose upper bound can weaken the regret guarantees of \cite{srinivas2010gaussian}. In practice, the parameter is set heuristically for each problem. 

We consider a set-up similar to the multi-armed bandits problem above. Consider the parameterized family of GP-UCB algorithms given by Algorithm \ref{alg:gp-ucb} with parameter $\sigma^2=s\in [s_{\min},s_{\max}]$ for some $0<s_{\min}<s_{\max}<\infty$. Let $\cD$ be a distribution over some  problems, i.e.\ a distribution over noisy (random) real-valued functions $f:\cC\rightarrow[0,H]$ with $\cC\subset\R^d$. It is typical to discretize the domain $\cC$ when computing the $\argmax$ in Algorithm \ref{alg:gp-ucb}, and usually $f(\cdot)$ is more expensive to evaluate than the UCB acquisition function $a_{t}(\x):={\mu}_{t}(\x)+\sqrt{\beta_t}\sigma_{t}(\x)$ for any point $\x$ on the finite discretization $\Tilde{\cC}$ of $\cC$ with $|\Tilde{\cC}|=n$. 

\begin{theorem}\label{thm:gp-sample-complexity-any-d} Consider the above setup for any arbitrary $\cD$. Let $n=|\Tilde{\cC}|$. For any $\epsilon,\delta>0$, $N$ problems $\{P_i,\mathbf{z}_i\}_{i=1}^N$ sampled from $\cD$ with $N=O\left(\left(\frac{H}{\epsilon}\right)^2\left(T\log nT +\log \frac{1}{\delta}\right)\right)$ are sufficient to ensure that with probability at least $1-\delta$, for all $s\in[s_{\min},s_{\max}]$, we have that $$\left\vert \frac{1}{N}\sum_{i=1}^N l_T^{P_i, \mathbf{z}_i}(s)-\bbE_{P\sim\cD}l_T^{P}(s)\right\vert<\epsilon.$$ 
\end{theorem}
\begin{proof}
Fix a problem  $P\in\Pi$. Fix the random coins $\mathbf{z}$ used to draw the arm rewards according to $\cD_P$ for the $T$ rounds. 
Let $t\in[T]$ denote some round of the online game, and suppose we are running Algorithm \ref{alg:gp-ucb} with noise parameter $s$.  \looseness-1

For a fixed sequence $\x_1,\dots,\x_t$, observe that each entry of $(\cK_t+sI)^{-1}$ is a rational function with degree at most $t$ in $s$  (e.g.\ Lemma C.2, \cite{balcan2022provably}). Thus, $\mu_t(\x)$ and $\sigma_t(\x)$ have each coordinate a  rational  function (ratio of two polynomials)  of $s$ of degree at most $t$. $\beta_t$ is typically set independently of $s$ \cite{srinivas2010gaussian}. Therefore, the acquisition function $a_t(\x)$ is a rational  function  of $s$ with degree at most $t$.
For any two points in the discretization $\x_i,\x_j\in \Tilde{C}$, $a_t(\x_i)\ge a_t(\x_j)$ corresponds to at most $2t$ disjoint intervals of $s$. Thus over all choices of $\x_i,\x_j$, there are at most $2t\cdot{n \choose 2}\le n^2t$ intervals of $s$ over each of which $\x_{t+1}$ is fixed, provided $\x_1,\dots,\x_t$ are fixed. By induction over $t$, there are at most $n^{2T}\cdot T!$ intervals of $s$ over which $\x_1,\dots,\x_T$ and therefore the behavior of Algorithm \ref{alg:gp-ucb} is fixed. By Lemma 2.3 of \cite{balcan2020data}, the pseudo-dimension of the class of loss functions is at most $O(\log(n^{2T}T!))=O(T\log nT)$. The sample complexity bound  follows from classic results (Chapter 19, \cite{anthony1999neural}).
\end{proof}

\noindent We find that the sample complexity required to accurately learn $s$ increases linearly with $T$. But this does not pose a significant limitation in applications such as hyperparameter optimization in deep learning, where $T$ is often very small.

\section{Experiments}\label{sec:expts}

\begin{figure*}[t]
\centering{\includegraphics[width=0.95\textwidth]{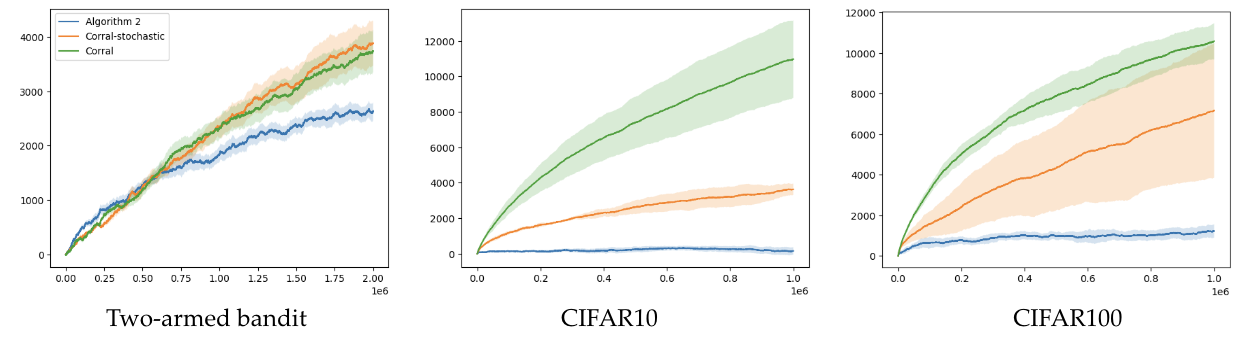}}\vspace{-0.02in}

    \caption{Comparison of Algorithm \ref{alg:tuned-ucb} to corralling based algorithms \textsc{Corral} \cite{agarwal2017corralling} and \textsc{Corral-stochastic} \cite{arora2021corralling}.}
    \label{fig:corral}
\end{figure*}

In this section, we provide empirical evidence for the significance of our hyperparameter transfer framework on real and synthetic data. As baselines, we consider corralling-based algorithms which are quite popular for learning bandit hyperparameters. As described previously, these approaches work by constructing a (finite) band of bandit algorithms corresponding to a grid of hyperparameter values, and running a meta-algorithm for selecting the hyperparameter in each round. The original \textsc{Corral} algorithm  \cite{agarwal2017corralling} uses an OMD (Online Mirror Descent) meta-algorithm with Log-Barrier regularizer, and the \textsc{Corral-stochastic} algorithm \cite{arora2021corralling} uses a Tsallis-INF regularizer and achieves stronger instance-dependent regret guarantees for stochastic bandits.
By exploiting offline data, we out-perform both these corralling-based approaches on real datasets involving tuning the learning rate of neural networks on benchmark datasets.\looseness-1

\textbf{Synthetic two-armed bandits.} We consider a simple two-armed bandits problem with Bernoulli arm rewards (see Appendix for more experiments on uniform and Gaussian rewards). Arm 1 draws a reward of 0 or 1 with probability 0.5 each in all tasks. Arm 2 draws a reward of value 1 with probability $0.5+\epsilon$ with $\epsilon\sim\cN(0.01,\sigma_b^2=0.01)$ and 0 otherwise. Given the small arm gap, this is a challenging problem that needs a lot of exploration.

\textbf{Hyperparameter tuning for Deep Learning.} We also consider the task of tuning the learning rate for training neural networks on  image classification tasks. The arms consist of 11 different learning rates $(0.001, 0.002, 0.004, 0.006, 0.008, 0.01, 0.05, 0.1, 0.2, 0.4, 0.8)$ and the arm reward is given by the classification accuracy of feedforward neural networks trained via SGD (stochastic gradient descent) with that learning rate and a batch size of 64 for 20 epochs. We present our results for CIFAR-10 and CIFAR-100~\cite{Krizhevsky09learning} benchmark image classification datasets. The task distribution
is defined by a uniform distribution over the label noise
proportions $(0.0, 0.1, 0.2, 0.3)$, and the network depth
$(2, 4, 6, 8, 10)$. All our experiments on CIFAR are run on $1$ Nvidia A100 GPU. 

\textbf{Setup and Discussion.} For each dataset we run Algorithm \ref{alg:tuned-ucb} over $N=200$ training/offline tasks with time horizon $T_o=20$, and run corralling for a grid of ten hyperparameter values $\alpha=\{0.1, 0.2, 0.5, 1, 2, 5, 10, 20, 50, 100\}$. Figure \ref{fig:corral} compares the effectiveness of running UCB with the learned hyperparameter $\hat{\alpha}$ vs.\ corralling over a grid of hyperparameters over 1000000 time steps (mean regret and standard deviation over 5 iterations). Our algorithm which exploits offline data significantly outperforms corralling based algorithms on real datasets. The key advantage of our approach is that by learning a good hyperparameter offline, we can save significantly on exploration. We verify this hypothesis by considering a synthetic two-armed bandits problem with Bernoulli parameters $0.5$ and roughly $0.51$, where a large amount of exploration is unavoidable. Even on this challenging synthetic dataset where the arm rewards are extremely close and more exploration is needed, our algorithm beats corralling over a longer time horizon of 2000000 steps. Further details and addtitional experiments showing dependence of regret on $\alpha$ and the number of training tasks $N$ are in the Appendix, where we also   empirically estimate $\cQ_{\cD}$ and show that typical values on natural problems are quite small, implying that our proposed algorithms are sample and computationally efficient in practice. Recall that our inter-task, intra-task and computational complexities scale as $O(\log \cQ_{\cD})$, $O(\cQ_{\cD})$, $O(\cQ_{\cD})$ respectively.

\section{Conclusion, Limitations and Future Work}
We study the problem of tuning hyperparameters of stochastic bandit algorithms, given access to offline data. Our setting is motivated by large information theoretic gaps for identifying the best hyperparameter in a fully online fashion, in the bandit setting. We provide a formal framework where the tasks are drawn iid from some distribution and the learner has access to some offline (training) tasks. For tuning the exploration parameter in UCB and the noise parameter in GP-UCB, we obtain bounds on the time horizon and number of the offline tasks needed to obtain any desired generalization performance on the unseen test tasks from the same distribution.\looseness-1

We believe the intra-task sample complexity bounds provided in our work can be improved with more careful arguments. An important question  that our work doesn't yet address is: \emph{how to strategically collect offline data to minimize the intra-task sample complexity?} 
Another direction is to tune hyperparameters beyond UCB-style algorithms.


\bibliographystyle{alpha}
\bibliography{main}

\newcommand{\etalchar}[1]{$^{#1}$}
\begin{thebibliography}{PPAY{\etalchar{+}}20}

\bibitem[AB99]{anthony1999neural}
Martin Anthony and Peter~L Bartlett.
\newblock {\em Neural network learning: Theoretical foundations}, volume~9.
\newblock Cambridge University Press, 1999.

\bibitem[ABG{\etalchar{+}}20]{angermueller2020population}
Christof Angermueller, David Belanger, Andreea Gane, Zelda Mariet, David Dohan, Kevin Murphy, Lucy Colwell, and D~Sculley.
\newblock Population-based black-box optimization for biological sequence design.
\newblock In {\em International Conference on Machine Learning}, pages 324--334. PMLR, 2020.

\bibitem[ACBF02]{auer2002finite}
Peter Auer, Nicolo Cesa-Bianchi, and Paul Fischer.
\newblock Finite-time analysis of the multiarmed bandit problem.
\newblock {\em Machine Learning}, 47(2):235--256, 2002.

\bibitem[AKGK23]{azizi2023meta}
Javad Azizi, Branislav Kveton, Mohammad Ghavamzadeh, and Sumeet Katariya.
\newblock Meta-learning for simple regret minimization.
\newblock In {\em Proceedings of the AAAI Conference on Artificial Intelligence}, volume~37, pages 6709--6717, 2023.

\bibitem[ALB13]{azar2013sequential}
Mohammad~Gheshlaghi Azar, Alessandro Lazaric, and Emma Brunskill.
\newblock Sequential transfer in multi-armed bandit with finite set of models.
\newblock In {\em Advances in Neural Information Processing Systems}, 2013.

\bibitem[ALNS17]{agarwal2017corralling}
Alekh Agarwal, Haipeng Luo, Behnam Neyshabur, and Robert~E Schapire.
\newblock Corralling a band of bandit algorithms.
\newblock In {\em Conference on Learning Theory}, pages 12--38. PMLR, 2017.

\bibitem[AMM21]{arora2021corralling}
Raman Arora, Teodor~Vanislavov Marinov, and Mehryar Mohri.
\newblock Corralling stochastic bandit algorithms.
\newblock In {\em International Conference on Artificial Intelligence and Statistics}, pages 2116--2124. PMLR, 2021.

\bibitem[AMS09]{audibert2009exploration}
Jean-Yves Audibert, R{\'e}mi Munos, and Csaba Szepesv{\'a}ri.
\newblock Exploration--exploitation tradeoff using variance estimates in multi-armed bandits.
\newblock {\em Theoretical Computer Science}, 410(19):1876--1902, 2009.

\bibitem[AYPS11]{abbasi2011improved}
Yasin Abbasi-Yadkori, D{\'a}vid P{\'a}l, and Csaba Szepesv{\'a}ri.
\newblock Improved algorithms for linear stochastic bandits.
\newblock {\em Advances in Neural Information Processing Systems}, 24, 2011.

\bibitem[Bal20]{balcan2020data}
Maria-Florina Balcan.
\newblock Book chapter {Data-Driven Algorithm Design}.
\newblock In {\em Beyond Worst Case Analysis of Algorithms, T. Roughgarden (Ed)}. Cambridge University Press, 2020.

\bibitem[BBBK11]{bergstra2011algorithms}
James Bergstra, R{\'e}mi Bardenet, Yoshua Bengio, and Bal{\'a}zs K{\'e}gl.
\newblock Algorithms for hyper-parameter optimization.
\newblock {\em Advances in Neural Information Processing Systems}, 24, 2011.

\bibitem[BBSZ23]{balcan2023analysis}
Maria-Florina Balcan, Avrim Blum, Dravyansh Sharma, and Hongyang Zhang.
\newblock An analysis of robustness of non-{L}ipschitz networks.
\newblock {\em Journal of Machine Learning Research}, 24(98):1--43, 2023.

\bibitem[BDD{\etalchar{+}}21]{balcan2021much}
Maria-Florina Balcan, Dan DeBlasio, Travis Dick, Carl Kingsford, Tuomas Sandholm, and Ellen Vitercik.
\newblock How much data is sufficient to learn high-performing algorithms? {G}eneralization guarantees for data-driven algorithm design.
\newblock In {\em Proceedings of the 53rd Annual ACM SIGACT Symposium on Theory of Computing (STOC)}, pages 919--932, 2021.

\bibitem[BDS20]{sharma2020learning}
Maria-Florina Balcan, Travis Dick, and Dravyansh Sharma.
\newblock Learning piecewise {L}ipschitz functions in changing environments.
\newblock In {\em International Conference on Artificial Intelligence and Statistics}, pages 3567--3577. PMLR, 2020.

\bibitem[BDS21]{blum2021learning}
Avrim Blum, Chen Dan, and Saeed Seddighin.
\newblock Learning complexity of simulated annealing.
\newblock In {\em International Conference on Artificial Intelligence and Statistics}, pages 1540--1548. PMLR, 2021.

\bibitem[BDV18]{balcan2018dispersion}
Maria-Florina Balcan, Travis Dick, and Ellen Vitercik.
\newblock Dispersion for data-driven algorithm design, online learning, and private optimization.
\newblock In {\em 2018 IEEE 59th Annual Symposium on Foundations of Computer Science (FOCS)}, pages 603--614. IEEE, 2018.

\bibitem[BDW18]{balcan2018data}
Maria-Florina Balcan, Travis Dick, and Colin White.
\newblock Data-driven clustering via parameterized {L}loyd's families.
\newblock {\em Advances in Neural Information Processing Systems}, 31, 2018.

\bibitem[BIW22]{bartlett2022generalization}
Peter Bartlett, Piotr Indyk, and Tal Wagner.
\newblock Generalization bounds for data-driven numerical linear algebra.
\newblock In {\em Conference on Learning Theory}, pages 2013--2040. PMLR, 2022.

\bibitem[BKST22]{balcan2022provably}
Maria-Florina Balcan, Mikhail Khodak, Dravyansh Sharma, and Ameet Talwalkar.
\newblock Provably tuning the elasticnet across instances.
\newblock {\em Advances in Neural Information Processing Systems}, 2022.

\bibitem[BM02]{bartlett2002rademacher}
Peter~L Bartlett and Shahar Mendelson.
\newblock Rademacher and {G}aussian complexities: {R}isk bounds and structural results.
\newblock {\em Journal of Machine Learning Research}, 3(Nov):463--482, 2002.

\bibitem[BNVW17]{balcan2017learning}
Maria-Florina Balcan, Vaishnavh Nagarajan, Ellen Vitercik, and Colin White.
\newblock Learning-theoretic foundations of algorithm configuration for combinatorial partitioning problems.
\newblock In {\em Conference on Learning Theory}, pages 213--274. PMLR, 2017.

\bibitem[BS21]{balcan2021data}
Maria-Florina Balcan and Dravyansh Sharma.
\newblock Data driven semi-supervised learning.
\newblock {\em Advances in Neural Information Processing Systems}, 34:14782--14794, 2021.

\bibitem[BS24]{balcan2024learning}
Maria~Florina Balcan and Dravyansh Sharma.
\newblock Learning accurate and interpretable decision trees.
\newblock In {\em The 40th Conference on Uncertainty in Artificial Intelligence}, 2024.

\bibitem[CB17]{cutkosky2017stochastic}
Ashok Cutkosky and Kwabena~A Boahen.
\newblock Stochastic and adversarial online learning without hyperparameters.
\newblock {\em Advances in Neural Information Processing Systems}, 30, 2017.

\bibitem[CDP20]{cutkosky2020upper}
Ashok Cutkosky, Abhimanyu Das, and Manish Purohit.
\newblock Upper confidence bounds for combining stochastic bandits.
\newblock {\em arXiv preprint arXiv:2012.13115}, 2020.

\bibitem[CGM{\etalchar{+}}13]{cappe2013kullback}
Olivier Capp{\'e}, Aur{\'e}lien Garivier, Odalric-Ambrym Maillard, R{\'e}mi Munos, and Gilles Stoltz.
\newblock Kullback-{L}eibler upper confidence bounds for optimal sequential allocation.
\newblock {\em The Annals of Statistics}, pages 1516--1541, 2013.

\bibitem[DKL{\etalchar{+}}22]{ding2022syndicated}
Qin Ding, Yue Kang, Yi-Wei Liu, Thomas Chun~Man Lee, Cho-Jui Hsieh, and James Sharpnack.
\newblock Syndicated bandits: A framework for auto tuning hyper-parameters in contextual bandit algorithms.
\newblock {\em Advances in Neural Information Processing Systems}, 35:1170--1181, 2022.

\bibitem[Efr92]{efron1992bootstrap}
Bradley Efron.
\newblock Bootstrap methods: another look at the jackknife.
\newblock In {\em Breakthroughs in statistics: Methodology and distribution}, pages 569--593. Springer, 1992.

\bibitem[FGMZ20]{foster2020adapting}
Dylan~J Foster, Claudio Gentile, Mehryar Mohri, and Julian Zimmert.
\newblock Adapting to misspecification in contextual bandits.
\newblock {\em Advances in Neural Information Processing Systems}, 33:11478--11489, 2020.

\bibitem[FKL19]{foster2019model}
Dylan~J Foster, Akshay Krishnamurthy, and Haipeng Luo.
\newblock Model selection for contextual bandits.
\newblock {\em Advances in Neural Information Processing Systems}, 32, 2019.

\bibitem[FRKL19]{finn2019online}
Chelsea Finn, Aravind Rajeswaran, Sham Kakade, and Sergey Levine.
\newblock Online meta-learning.
\newblock In {\em International Conference on Machine Learning}, pages 1920--1930. PMLR, 2019.

\bibitem[GC11]{garivier2011kl}
Aur{\'e}lien Garivier and Olivier Capp{\'e}.
\newblock The {KL-UCB} algorithm for bounded stochastic bandits and beyond.
\newblock In {\em Proceedings of the 24th annual conference on learning theory}, pages 359--376. JMLR Workshop and Conference Proceedings, 2011.

\bibitem[GR17]{gupta2017pac}
Rishi Gupta and Tim Roughgarden.
\newblock A {PAC} approach to application-specific algorithm selection.
\newblock {\em SIAM Journal on Computing}, 46(3):992--1017, 2017.

\bibitem[HNvdH23]{husain2024distributionally}
Hisham Husain, Vu~Nguyen, and Anton van~den Hengel.
\newblock {Distributionally Robust Bayesian Optimization} with $\varphi$-divergences.
\newblock {\em Advances in Neural Information Processing Systems (NeurIPS)}, 36, 2023.

\bibitem[Ito21]{ito2021parameter}
Shinji Ito.
\newblock Parameter-free multi-armed bandit algorithms with hybrid data-dependent regret bounds.
\newblock In {\em Conference on Learning Theory}, pages 2552--2583. PMLR, 2021.

\bibitem[KA21]{krishnamurthy2021optimal}
Sanath~Kumar Krishnamurthy and Susan Athey.
\newblock Optimal model selection in contextual bandits with many classes via offline oracles.
\newblock {\em arXiv preprint arXiv:2106.06483}, 2021.

\bibitem[KBT19]{khodak2019adaptive}
Mikhail Khodak, Maria-Florina Balcan, and Ameet Talwalkar.
\newblock Adaptive gradient-based meta-learning methods.
\newblock {\em Advances in Neural Information Processing Systems}, 32, 2019.

\bibitem[Kea95]{kearns1995bound}
Michael Kearns.
\newblock A bound on the error of cross validation using the approximation and estimation rates, with consequences for the training-test split.
\newblock {\em Advances in Neural Information Processing Systems}, 8, 1995.

\bibitem[KHL24]{kang2024online}
Yue Kang, Cho-Jui Hsieh, and Thomas Lee.
\newblock Online continuous hyperparameter optimization for generalized linear contextual bandits.
\newblock {\em Transactions on Machine Learning Research}, 2024.

\bibitem[KMH{\etalchar{+}}20]{kveton2020meta}
Branislav Kveton, Martin Mladenov, Chih-Wei Hsu, Manzil Zaheer, Csaba Szepesvari, and Craig Boutilier.
\newblock Meta-learning bandit policies by gradient ascent.
\newblock {\em arXiv preprint arXiv:2006.05094}, 2020.

\bibitem[KOH{\etalchar{+}}23]{khodak2023meta}
Mikhail Khodak, Ilya Osadchiy, Keegan Harris, Maria-Florina Balcan, Kfir~Y Levy, Ron Meir, and Zhiwei~Steven Wu.
\newblock Meta-learning adversarial bandit algorithms.
\newblock {\em Advances in Neural Information Processing Systems}, 2023.

\bibitem[Kol01]{koltchinskii2001rademacher}
Vladimir Koltchinskii.
\newblock Rademacher penalties and structural risk minimization.
\newblock {\em IEEE Transactions on Information Theory}, 47(5):1902--1914, 2001.

\bibitem[LCLS10]{li2010contextual}
Lihong Li, Wei Chu, John Langford, and Robert~E Schapire.
\newblock A contextual-bandit approach to personalized news article recommendation.
\newblock In {\em Proceedings of the 19th International Conference on World Wide Web}, pages 661--670, 2010.

\bibitem[LWB{\etalchar{+}}07]{lizotte2007automatic}
Daniel~J Lizotte, Tao Wang, Michael~H Bowling, Dale Schuurmans, et~al.
\newblock Automatic gait optimization with gaussian process regression.
\newblock In {\em IJCAI}, volume~7, pages 944--949, 2007.

\bibitem[LZZZ22]{luo2022corralling}
Haipeng Luo, Mengxiao Zhang, Peng Zhao, and Zhi-Hua Zhou.
\newblock Corralling a larger band of bandits: A case study on switching regret for linear bandits.
\newblock In {\em Conference on Learning Theory}, pages 3635--3684. PMLR, 2022.

\bibitem[Mas00]{massart2000some}
Pascal Massart.
\newblock Some applications of concentration inequalities to statistics.
\newblock In {\em Annales de la Facult{\'e} des sciences de Toulouse: Math{\'e}matiques}, volume~9, pages 245--303, 2000.

\bibitem[MMT{\etalchar{+}}22]{mate2022field}
Aditya Mate, Lovish Madaan, Aparna Taneja, Neha Madhiwalla, Shresth Verma, Gargi Singh, Aparna Hegde, Pradeep Varakantham, and Milind Tambe.
\newblock Field study in deploying restless multi-armed bandits: Assisting non-profits in improving maternal and child health.
\newblock In {\em Proceedings of the AAAI Conference on Artificial Intelligence}, volume~36, pages 12017--12025, 2022.

\bibitem[MNSR18]{mukherjee2018efficient}
Subhojyoti Mukherjee, KP~Naveen, Nandan Sudarsanam, and Balaraman Ravindran.
\newblock Efficient-{UCBV}: An almost optimal algorithm using variance estimates.
\newblock In {\em Proceedings of the AAAI Conference on Artificial Intelligence}, volume~32, 2018.

\bibitem[MZ21]{marinov2021pareto}
Teodor~Vanislavov Marinov and Julian Zimmert.
\newblock The pareto frontier of model selection for general contextual bandits.
\newblock {\em Advances in Neural Information Processing Systems}, 34:17956--17967, 2021.

\bibitem[PDG22]{pacchiano2022best}
Aldo Pacchiano, Christoph Dann, and Claudio Gentile.
\newblock Best of both worlds model selection.
\newblock {\em Advances in Neural Information Processing Systems}, 35:1883--1895, 2022.

\bibitem[PO06]{pandey2006handling}
Sandeep Pandey and Christopher Olston.
\newblock Handling advertisements of unknown quality in search advertising.
\newblock {\em Advances in Neural Information Processing Systems}, 19, 2006.

\bibitem[PPAY{\etalchar{+}}20]{pacchiano2020model}
Aldo Pacchiano, My~Phan, Yasin Abbasi~Yadkori, Anup Rao, Julian Zimmert, Tor Lattimore, and Csaba Szepesvari.
\newblock Model selection in contextual stochastic bandit problems.
\newblock {\em Advances in Neural Information Processing Systems}, 33:10328--10337, 2020.

\bibitem[PPM22]{peleg2022metalearning}
Amit Peleg, Naama Pearl, and Ron Meir.
\newblock Metalearning linear bandits by prior update.
\newblock In {\em International Conference on Artificial Intelligence and Statistics}, pages 2885--2926. PMLR, 2022.

\bibitem[SKKS10]{srinivas2010gaussian}
Niranjan Srinivas, Andreas Krause, Sham Kakade, and Matthias~W Seeger.
\newblock Gaussian process optimization in the bandit setting: No regret and experimental design.
\newblock In {\em International Conference on Machine Learning}, 2010.

\bibitem[SSA13]{swersky2013multi}
Kevin Swersky, Jasper Snoek, and Ryan~P Adams.
\newblock Multi-task bayesian optimization.
\newblock {\em Advances in Neural Information Processing Systems}, 26, 2013.

\bibitem[Sto74]{stone1974cross}
Mervyn Stone.
\newblock Cross-validatory choice and assessment of statistical predictions.
\newblock {\em Journal of the royal statistical society: Series B (Methodological)}, 36(2):111--133, 1974.

\bibitem[SXZ23]{shen2023wasserstein}
Yi~Shen, Pan Xu, and Michael Zavlanos.
\newblock Wasserstein distributionally robust policy evaluation and learning for contextual bandits.
\newblock {\em Transactions on Machine Learning Research (TMLR)}, 2023.

\bibitem[SZZB20]{si2020distributionally}
Nian Si, Fan Zhang, Zhengyuan Zhou, and Jose Blanchet.
\newblock Distributionally robust policy evaluation and learning in offline contextual bandits.
\newblock In {\em International Conference on Machine Learning (ICML)}, pages 8884--8894. PMLR, 2020.

\bibitem[TM17]{tewari2017ads}
Ambuj Tewari and Susan~A Murphy.
\newblock From ads to interventions: Contextual bandits in mobile health.
\newblock {\em Mobile health: sensors, analytic methods, and applications}, pages 495--517, 2017.

\bibitem[WDS{\etalchar{+}}24]{wang2022pre}
Zi~Wang, George~E Dahl, Kevin Swersky, Chansoo Lee, Zelda Mariet, Zachary Nado, Justin Gilmer, Jasper Snoek, and Zoubin Ghahramani.
\newblock Pre-trained {G}aussian processes for {B}ayesian optimization.
\newblock {\em Journal of Machine Learning Research}, 25(212):1--83, 2024.

\bibitem[WKK18]{wang2018regret}
Zi~Wang, Beomjoon Kim, and Leslie~P Kaelbling.
\newblock Regret bounds for meta bayesian optimization with an unknown {G}aussian process prior.
\newblock {\em Advances in Neural Information Processing Systems}, 31, 2018.

\bibitem[WWR23]{whitehouse2023improved}
Justin Whitehouse, Zhiwei~Steven Wu, and Aaditya Ramdas.
\newblock Improved self-normalized concentration in {H}ilbert spaces: Sublinear regret for {GP-UCB}.
\newblock {\em arXiv preprint arXiv:2307.07539}, 2023.

\bibitem[Yan07]{yang2007consistency}
Yuhong Yang.
\newblock Consistency of cross validation for comparing regression procedures.
\newblock {\em The Annals of Statistics}, pages 2450--2473, 2007.

\bibitem[YM14]{yogatama2014efficient}
Dani Yogatama and Gideon Mann.
\newblock Efficient transfer learning method for automatic hyperparameter tuning.
\newblock In {\em Artificial Intelligence and Statistics}, pages 1077--1085. PMLR, 2014.

\bibitem[ZS21]{zimmert2021tsallis}
Julian Zimmert and Yevgeny Seldin.
\newblock Tsallis-inf: An optimal algorithm for stochastic and adversarial bandits.
\newblock {\em Journal of Machine Learning Research}, 22(28):1--49, 2021.

\end{thebibliography}

\clearpage
\appendix

\section{Learning theory background}

For ready reference, we include in this section  some well-known concepts from statistical learning theory. In particular, we include the definitions of the Rademacher complexity and pseudo-dimension, and the corresponding well-known generalization guarantees. 

Our proof of Theorem 6.1 involves bounding the empirical Rademacher complexity which is generally used to distribution-dependent guarantees \cite{bartlett2002rademacher,koltchinskii2001rademacher}. 

\begin{definition}[Rademacher complexity] Let $\cF$ be a class of functions mapping $\cX\mapsto[0,H]$. The empirical Rademacher complexity of $\cF$ with respect to a sample $S = \{s_1,\dots,s_m\} \subseteq \cX$ of size $m$ is defined as
    $$\hat{R}(\cF,S):=\frac{1}{m}\bbE_{\sigma\sim\{-1,1\}^m}\left[\sup_{f\in\cF}\sum_{i=1}^m\sigma_if(s_i)\right],$$
    and the Rademacher complexity is defined as $${R}(\cF,\cD):=\bbE_{S\sim\cD^m}\hat{R}(\cF,S).$$
\end{definition}

The following generalization guarantee based on the Rademacher complexity is well-known.

\begin{theorem}\label{thm:rc-generalization}
    For any distribution $\cD$ over $\cX$, with probability at least $1 - \delta$ over the draw of $S = \{s_1,\dots,s_m\}\sim\cD^m$, for all $f\in\cF$, 
    $$\Bigg\lvert\bbE_{x\sim\cD}f(x)-\frac{1}{m}\sum_{i=1}^mf(x_i)\Bigg\rvert\le 2{R}(\cF,\cD)+H\sqrt{\frac{1}{m}\log\frac{1}{\delta}}.$$
\end{theorem}

In contrast, the pseudo-dimension is frequently used to analyze the learning theoretic complexity of real-valued  function classes and give worst case bounds over the distribution. The formal definition is stated here for convenience.

\begin{definition}[Shattering and pseudo-dimension \cite{anthony1999neural}]
Let $\cF$ be a set of functions mapping from $\cX$ to $\bbR$, and suppose that $S = \{x_1, \dots, x_m\} \subseteq \cX$. Then $S$ is pseudo-shattered by $\cF$ if there are real numbers $r_1, \dots, r_m$ such that for each $b \in \{0, 1\}^m$ there is a function $f_b$ in $\cF$ with $\mathrm{sign}(f_b(x_i) - r_i) = b_i$ for $i \in [m]$. We say that $r = (r_1, \dots, r_m)$ witnesses the shattering. We say that $\cF$ has pseudo-dimension $d$ if $d$ is the maximum cardinality of a subset $S$ of $\cX$ that is pseudo-shattered by $\cF$, denoted $\mathrm{Pdim}(\cF) = d$. If no such maximum exists, we say that $\cF$ has infinite pseudo-dimension. 
\end{definition}

Pseudo-dimension generalizes the notion of VC-dimension from binary valued functions to real-valued functions, and is another classic learning theoretic complexity notion. The following uniform convergence sample complexity for any function in class $\cF$ when $\mathrm{Pdim}(\cF)$ is finite is known. 
\begin{theorem}[Uniform convergence sample complexity via pseudo-dimension \cite{anthony1999neural}]\label{thm:pdim}
    Suppose $\cH$ is a class of real-valued functions with range in $[0, H]$ and finite $\mathrm{Pdim}(\cF)$. For every $\epsilon > 0$ and $\delta \in (0, 1)$, the sample complexity of $(\epsilon, \delta)$-uniformly learning the class $\cH$ is $O\left(\left(\frac{H}{\epsilon}\right)^2\left(\mathrm{Pdim}(\cF)\log\left(\frac{H}{\epsilon}\right) + \log\left(\frac{1}{\delta}\right) \right)\right)$.
\end{theorem}

Uniform learning is closely related to  PAC (probably approximately correct) learning. It is easy to see that $(\epsilon, \delta)$-uniform learning corresponds to $(\epsilon / 2, \delta)$-PAC learning.

\section{Proof of Theorem \ref{thm:impossibility_hp_tuning}}
\thmlb*

\noindent {\it Proof.} We first prove the following intermediate result which provides distribution dependent regret lower bounds for MABs.  Consider a $K$-armed MAB problem where the reward of each arm follows a Gaussian distribution with unknown mean and variance. The only information known about the variance is that it is bounded between $[0, B^2]$ for some very large $B$. We denote this model class by $\mathcal{M}$.
\begin{theorem}
    \label{thm:lb_mab_hp_tuning_gaussian}
    Let $\nu$ be any problem instance that belongs to the model class $\mathcal{M}$ described above. Let $\nu_i$ be the distribution of arm $i$ in problem instance $\nu$. Let $\mathbf{E}(\nu_i)$ be the mean of $\nu_i$ and $\mathbf{V}(\nu_i)$ be its variance. Without loss of generality, let $1$ be the index of the best arm. Let $\Delta_i = \mathbf{E}(\nu_1) - \mathbf{E}(\nu_i)$. Moreover, suppose $B > \max_{i>1} \Delta_i^2 + \mathbf{V}(\nu_i)$. Then the regret of any consistent algorithm satisfies the following instance dependent lower bound 
    \begin{align*}
        \lim_{T\to\infty}\frac{\text{Reg}_T(\nu)}{\log{T}} \geq \sum_{i: \Delta_i > 0} \frac{2\Delta_i}{\log\left(1+\frac{\Delta_i^2}{\mathbf{V}(\nu_i)}\right)}.
    \end{align*}
\end{theorem}
\begin{proof}
    From standard distribution dependent lower bounds~\cite{cappe2013kullback}, we know that 
    \begin{align*}
        \lim_{T\to\infty}\frac{\text{Reg}_T(\nu)}{\log{T}} \geq \sum_{i: \Delta_i > 0} \frac{\Delta_i}{\text{KL}_{\text{inf}}(\nu_i, \mathbf{E}(\nu_1))},
    \end{align*}
    where $\text{KL}_{\text{inf}}(\nu_i, \mathbf{E}(\nu_1)) = \inf_{\nu'\in\mathcal{M}} \{\text{KL}(\nu_i, \nu'): \mathbf{E}(\nu') > \mathbf{E}(\nu_1)\}.$ 
    Note that KL divergence between two univariate Gaussians is given by
    \[
    \text{KL}(\cN(\mu_1, \sigma_1^2)||\cN(\mu_2, \sigma_2^2)) = \log\frac{\sigma_2}{\sigma_1} + \frac{\sigma_1^2 + (\mu_1-\mu_2)^2}{2\sigma_2^2} - \frac{1}{2}.
    \]
    Using this, we have the following expression for $\text{KL}_{\text{inf}}$ of our model class
    \begin{align*}
        \text{KL}_{\text{inf}}(\nu_i, \mathbf{E}(\nu_1)) =\inf_{\sigma \in [0, B]}\frac{1}{2}\log\frac{\sigma^2}{\mathbf{V}(\nu_i)} + \frac{\mathbf{V}(\nu_i) + \Delta_i^2}{2\sigma^2} - \frac{1}{2}.
    \end{align*}
    The above objective is convex in $\sigma^{-1}$. It's minimizer is achieved at $\min\left\lbrace\sqrt{\Delta_i^2 + \mathbf{V}(\nu_i)}, B\right\rbrace.$ If $B$ is large, then the minimizer is achieved at $\sqrt{\Delta_i^2 + \mathbf{V}(\nu_i)}.$ Substituting this in the expression for $\text{KL}_{\text{inf}}$ gives us
    \begin{align*}
        \text{KL}_{\text{inf}}(\nu_i, \mathbf{E}(\nu_1)) = \frac{1}{2}\log\left(1+\frac{\Delta_i^2}{\mathbf{V}(\nu_i)}\right).
    \end{align*}
    Substituting this in the regret lower bound gives us the required result.
\end{proof}
\begin{remark}
\label{rem:ucb_known_sigma}
    Note that if the variances of all the arms are equal to $\sigma^2$ and known apriori, then the instance dependent regret lower bound is given by $\sum_{i: \Delta_i > 0} \frac{2\sigma^2}{\Delta_i}\log{T}$. Furthermore, we have algorithms such as UCB that can achieve this bound~\cite{cappe2013kullback}. 
\end{remark}
\noindent We now proceed to the proof of Theorem~4.1. Recall, our goal is to show that $\lim_{T\to\infty} R_T(\widetilde{A}; P)/R_T(A_{\rho^*}; P) > 1$. First, consider the numerator $R_T(\widetilde{A}; P)$. From Theorem~\ref{thm:lb_mab_hp_tuning_gaussian}, it is clear that there exist problem instances $P$ - with equal arm reward variances ($\mathbf{V}$) - in the model class $\mathcal{M}$ we have
\[
\lim_{T\to\infty}\frac{R_T(\widetilde{A}; P)}{\log{T}} \geq \sum_{i: \Delta_i > 0} \frac{2\Delta_i}{\log\left(1+\frac{\Delta_i^2}{\mathbf{V}}\right)},
\]
Now, consider the denominator $R_T(A_{\rho^*}; P)$. From Remark~\ref{rem:ucb_known_sigma} we know that there exists an exploration hyperparameter for UCB that gives us the following regret upper bound
\[
\lim_{T\to\infty}\frac{R_T(A_{\rho^*}; P)}{\log{T}} \leq \sum_{i: \Delta_i > 0} \frac{2\mathbf{V}}{\Delta_i}.
\]
Combining the above two inequalities and using the fact that $\log(1+x) \leq x$ for positive $x$, we get the required result.
\qed



\section{Proof of Theorem \ref{thm:sample-complexity-inter-task}}
\thmintertask*

\noindent{\it Proof.} 
Fix a problem  $P\in\Pi$. Fix the random coins $\mathbf{z}$ used to draw the arm rewards according to $D_P$ for the $T$ rounds. 
We will use the piecewise loss structure to bound the  Rademacher complexity, which would imply uniform convergence guarantees by applying standard learning-theoretic results. Let $\rho, \dots, \rho_m$ denote a collection of parameter values, with one parameter from each of the $m\le \sum_{i=1}^Nq(l_T^{P_i, \mathbf{z}_i}(\cdot))$ pieces of the dual class functions $l^{P_i,\mathbf{z}_i}_T(\cdot)$ for $i\in[N]$, i.e.\ across problems in the sample $\{P_1,\dots, P_N\}$ for some fixed randomizations. Let $\cF=\{f_{\rho}:(P,\mathbf{z})\mapsto l^{P,\mathbf{z}}_T(\rho)\mid \rho\in\cP\}$ be a family of functions on a given sample of instances $S=\{P_i,\mathbf{z}_i\}_{i=1}^N$. Since the function $f_\rho$ is constant on each of the $m$ pieces, we have the empirical Rademacher complexity,

\begin{align*}
    \hat{R}(\cF,S):&=\frac{1}{N}\bbE_\sigma\left[\sup_{f_\rho\in\cF}\sum_{i=1}^N\sigma_i f_\rho(P_i,\mathbf{z}_i)\right]\\
    &=\frac{1}{N}\bbE_\sigma\left[\sup_{j\in[m]}\sum_{i=1}^N\sigma_i f_{\rho_j}(P_i,\mathbf{z}_i)\right]\\
    &=\frac{1}{N}\bbE_\sigma\left[\sup_{j\in[m]}\sum_{i=1}^N\sigma_i v_{ij}\right],
\end{align*}
where $\sigma=(\sigma_1,\dots,\sigma_m)$ is a tuple of i.i.d.\ Rademacher random variables, and $v_{ij}:=f_{\rho_j}(P_i,\mathbf{z}_i)$. Note that $v^{(j)}:=(v_{1j},\dots, v_{Nj})\in[0,H]^N$, and therefore $||v^{(j)}||_2\le H\sqrt{N}$, for all $j\in[m]$. An application of Massart's lemma \cite{massart2000some} gives 

\begin{align*}
    \hat{R}(\cF,S)&=\frac{1}{N}\bbE_\sigma\left[\sup_{j\in[m]}\sum_{i=1}^N\sigma_i v_{ij}\right] \\&\le H\sqrt{\frac{2\log m}{N}}\\&\le H\sqrt{\frac{2\log\sum_{i=1}^Nq(l_T^{P_i, \mathbf{z}_i}(\cdot))}{N}}.
\end{align*}

\noindent Taking average over $S$, gives the Rademacher complexity

\begin{align*}
    {R}(\cF,\cD)=\bbE_{S}\hat{R}(\cF,S) &\le \bbE_{S} H\sqrt{\frac{2\log\sum_{i=1}^Nq(l_T^{P_i, \mathbf{z}_i}(\cdot))}{N}}\\
    &\le H\sqrt{\bbE_{S}\frac{2\log\sum_{i=1}^Nq(l_T^{P_i, \mathbf{z}_i}(\cdot))}{N}}\\
    &\le H\sqrt{\frac{2\log\bbE_{S}\sum_{i=1}^Nq(l_T^{P_i, \mathbf{z}_i}(\cdot))}{N}},
\end{align*}

\noindent where we have twice used the Jensen's inequality. Finally, using Definition 1, we obtain the following bound,

\begin{align*}
    {R}(\cF,\cD)\le H\sqrt{\frac{2\log N + 2\log Q_\cD}{N}}.
\end{align*}

\noindent Standard Rademacher complexity bounds [Barlett et al.\ 2002] now imply the desired sample complexity bound.
\qed

\section{Proof of Theorem \ref{thm:abundant_offline_data}}
\thmintratask*
\begin{proof}
    The hyperparameter tuning algorithm is the Empirical Risk Minimizer (ERM) $$\min_{\rho}\frac{1}{N}\sum_{i=1}^N l_T^{P_i, \mathbf{z}_i}(\rho)$$
    for which the error is at most
    \begin{align*}
        \bbE_{P\sim\cD}[l_T^{P}(\hat\rho)] - \min_{\rho} \bbE_{P\sim\cD}[l_T^{P}(\rho)] &\leq 2H\sqrt{\frac{2\log N + 2\log Q_\cD}{N}}+H\sqrt{\frac{1}{N}\log\frac{1}{\delta}}\\
        &\leq 4H\sqrt{\frac{\log \frac{N}{\delta} + \log Q_\cD}{N}},
    \end{align*}

    \noindent following the bound on the Rademacher complexity in the proof of Theorem 6.1 and using Theorem \ref{thm:rc-generalization}. We need sufficiently large offline runs to compute $l_T^{P_i, \mathbf{z}_i}(\rho)$ for each task in order to implement the ERM. 
    
    If $n\le Q_\cD$, we simply pull each arm $T$ times (or use a stochastic policy) which ensures that loss for all possible arm pull sequences and therefore for any value of $\rho$ can be computed for the given task. Clearly $\bbE[T_o]\le nT$ in this case.

    If $n>Q_\cD$, we exploit the piecewise constant structure of $l_T^{P_i, \mathbf{z}_i}(\rho)$ to sequentially compute the pieces. For any single piece, we need at most $T$ arm pulls to compute the loss over the piece and there are $Q_\cD$ pieces in expectation, per Definition 1. Therefore $\bbE[T_o]\le Q_{\cD}T$ in this case.
\end{proof}

\section{Proof of Theorem \ref{thm:ucb}}
\thmucb*

\begin{proof}
    Fix a problem  $P\in\Pi$. Fix the random coins $\mathbf{z}$ used to draw the arm rewards according to $\cD_P=\cD_1\times\dots\times\cD_n$ for the $T$ rounds, i.e.\ suppose the environment draws the $T$ rewards $\mathbf{r}_a=(r_a[1],\dots,r_a[T])$ according to $\cD_a$ for each arm $a\in[n]$ beforehand, and reveals $\mathbf{r}_{a_t}[t_{a_t}]$ if the learner selects arm $a_t$ at time $t$, where $t_{a_t}$ is the number of pulls of arm $a_t$ so far. 
Let $t\in[T]$ denote some round of the online game, and suppose we are running Algorithm 1 with (exploration) parameter $\alpha$.

The algorithm is identical for all $\alpha$ for $t\le n$. In round $t>n$, the behavior of Algorithm 1 can only change at critical points of the form $\sqrt{\alpha}=\frac{1}{\sqrt{\log t}}\cdot\frac{\hat{\mu}_{i}-\hat{\mu}_{j}}{\frac{1}{\sqrt{t_{j}}}-\frac{1}{\sqrt{t_{i}}}}$ for some $1\le i\ne j\le n$ with $t_i\ne t_j$. This quantity is fully determined by the vector $\mathbf{t}=(t_1,\dots,t_n)$ given fixed randomization $\z$. By a simple combinatorial argument, $\mathbf{t}$ takes at most ${n+(t-n)-1 \choose n-1}={t-1\choose n-1}=O((t-1)^{n-1})$. Thus, there are at most $(t-1)^{n-1}$ distinct critical points in the derandomized dual function $l^{P,\mathbf{z}}_T(\alpha)$ corresponding to time step $t$. Adding up over all time steps, we have at most $(T-n)(T-1)^{n-1}\le T^{n-1}$ critical points in $l^{P,\mathbf{z}}_T(\alpha)$, therefore $\log Q_\cD=O(n\log T)$ by definition.
\end{proof}

\section{Proof of Theorem \ref{thm:ucb-categorical}}

\thmucbcategorial*
\noindent{\it Proof.} 
    As in the proof of Theorem 7.1, we fix a problem $P\in\Pi$ and random coins $\mathbf{z}$. The algorithm is identical for all $\alpha$ for $t\le n$. In round $t>n$, the behavior of Algorithm 1 can only change at critical points of the form $\sqrt{\alpha}=\frac{1}{\sqrt{\log t}}\cdot\frac{\hat{\mu}_{i}-\hat{\mu}_{j}}{\frac{1}{\sqrt{t_{j}}}-\frac{1}{\sqrt{t_{i}}}}$ for some $1\le i\ne j\le n$ with $t_i\ne t_j$. In this expression, $t$ can take at most $T-n$ distinct values, and $1\le t_i,t_j \le T-n+1$. For any arm $i$, the total reward in $t_i$ arm pulls is at most $(K-1)t_i$ (i.e.\ at most $KT$ distinct values), and the mean reward $\hat{\mu}_i$ has at most $KT^2$ distinct possible values. Overall, the number of distinct possible values of critical points in the derandomized dual $l^{P,\mathbf{z}}_T(\alpha)$ is at most 
$(T-n)\cdot (T-n+1)^2 \cdot (KT^2)^2\le K^2T^7$. The result follows from Definition 1.
\qed

\section{Proof of Theorem \ref{thm:linucb}}
\thmlinucb*

\noindent{\it Proof.} 
    Fix a problem  $P\in\Pi$. Fix the random coins $\mathbf{z}$ used to draw the arm rewards according to $\cD_P$ for the $T$ rounds. In round $t$, the behavior of Algorithm 2 can only change at critical points of the form 
    ${\alpha}=\frac{\hat{\theta}_{t}^T(x_{t,i}-{x}_{t,j})}{\sqrt{x_{t,j}^TK^{-1}x_{t,j}}-\sqrt{x_{t,i}^TK^{-1}x_{t,i}}}$ for some $1\le i\ne j\le n$. This quantity is fully determined by the vector $\mathbf{t}=(t_1,\dots,t_n)$ given fixed randomization $\mathbf{z}$. 
    Observe  that these critical points can partition the parameter domain $\cP$ into at most $n$ intervals, since sets of points where a given arm dominates is a continuous set. In any  round $t>n$, any of the intervals from the previous round can only be subdivided into at most $n$ pieces (same argument as above applied to any interval over which the algorithm behavior is identical for all the rounds till the $(t-1)$-th round) where the algorithm behaves identically. At the end of $T$ rounds, we have at most $n^{T}$ such intervals. Therefore $l^{P,\mathbf{z}}_T(\alpha)$ is piecewise constant with at most $n^T$ pieces, or $Q_{\cD}=O(T\log n)$ by definition.
\qed

\section{Proof of Theorem \ref{thm:sample-complexity-with-prior}}
\thmucbprior*
\noindent{\it Proof.} 
    We will employ tools due to \cite{bartlett2022generalization} for bounding the pseudo-dimension of the loss function class. We give a GJ algorithm (Definition 3.1, \cite{bartlett2022generalization}) for computing the loss function $l_T^{P_i, \mathbf{z}_i}(\alpha,\mathbf{\hat{\mu}^0})$ on any given fixed instance and randomization and bound its degree and predicate complexity.
    
    To compute $l_T^{P_i, \mathbf{z}_i}(\alpha,\mathbf{\hat{\mu}^0})$ it is sufficient to determine which arm $A_t$ is pulled in all the rounds $t\in[T]$ for any given parameters $\alpha,\mathbf{\hat{\mu}^0}$. 
    The $\mathrm{argmax}$ in the expression for $A_t$ can be determined by evaluating predicates of the form $$(\alpha \log t)\left(\frac{1}{\sqrt{t_j}}-\frac{1}{\sqrt{t_i}}\right)^2\gtrless \left(\frac{\mathbf{\hat{\mu}^0}_i+\hat{\mu}^{1:t}_i}{t_i}-\frac{\mathbf{\hat{\mu}^0}_j+\hat{\mu}^{1:t}_j}{t_j}\right)^2$$ for each $i,j\in[n]$, where
    $\hat{\mu}^{1:t}_i$ denotes the arm rewards collected by Algorithm 1 (with prior) for arm $i$ in rounds 1 through $t$. For a fixed randomization $\mathbf{z}$, there are at most $n^t$ distinct possibilities for $\{\hat{\mu}^{1:t}_1,\dots,\hat{\mu}^{1:t}_n\}$. Moreover, there are at most $t^2$ distinct possibilities for $t_i,t_j$. Summing up over $t$, we have at most $\sum_{t=1}^Tt^2n^t=O(T^3n^T)$ distinct predicates. To ensure rational computations, we introduce $O(T)$ `constant' inputs $\log t$ and $\sqrt{t}$ for each $t\in [T]$, in addition to the $n+1$ hyperparameters. The degree of any predicate is at most 4. 
    
    By Theorem 3.3 of [Barlett et al.\ 2022], the pseudodimension of the loss function class is 
    $O((n+1+T)\cdot\log(4\cdot T^3n^T))=O((n+T)T\log n)$. The sample complexity bound now follows from standard results (Chapter 19,  \cite{anthony1999neural}).
\qed

\section{Further Details for Setup of Experiments in the Main Body}

\begin{figure}[t]
\centering{\includegraphics[scale=0.5]{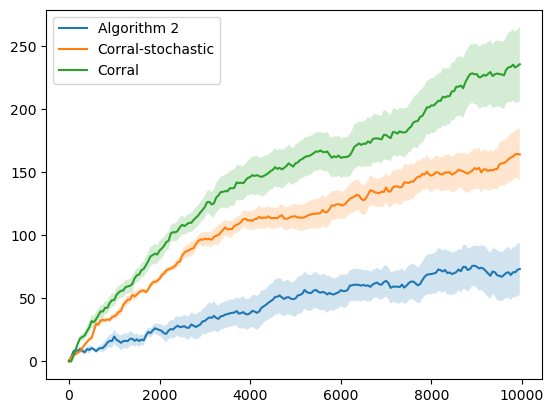}}

    \caption{Comparison of Algorithm \ref{alg:tuned-ucb} to corralling based algorithms \textsc{Corral} \cite{agarwal2017corralling} and \textsc{Corral-stochastic} \cite{arora2021corralling} on 2-arm bandits with large arm-reward gap.}
    \label{fig:corral2}
\end{figure}


{\bf More details on Baselines.} Our baselines include the original \textsc{Corral} algorithm (for both stochastic and adversarial bandits, with regret bound $O(\sqrt{MT})$ where $M$ is the number of base algorithms) of \cite{agarwal2017corralling} and the improvement \textsc{Corral-stochastic} for stochastic bandits (with instance-dependent regret bounds) due to \cite{arora2021corralling}. We run both the algorithms with the default theory-suggested parameters e.g.\ step-size parameters $\eta$ and $\beta$ in \textsc{Corral} are set according to Theorem  5 and Algorithm 1 in \cite{agarwal2017corralling}.  A deviation from the theory is that we never restart the corralling algorithms, similar to previous experiments by \cite{arora2021corralling}. The regret bounds from theory remain meaningful for this variant in practice.

{\bf Details of empirical runs.} Since corralling methods can only be implemented for a finite number of base algorithms, we consider a finite grid of ten $\alpha$ values $\{0.1, 0.2, 0.5, 1, 2, 5, 10, 20, 50, 100\}$ in all our experiments. Note that this is a limitation of corralling which our approach overcomes (we can learn over a continuum of hyperparameters), but for fair empirical comparison we learn the best hyperparameter using our algorithm on the same grid from similar training tasks. For the synthetic two-arms dataset, different tasks are generated by randomly sampling the mean reward for the second arm according to $\cN(0.51, 0.01)$ while keeping the mean reward for the first arm fixed at $0.5$. For the CIFAR experiments, we sample tasks according to a uniform distribution over  the label noise
proportions $(0.0, 0.1, 0.2, 0.3)$, and the network depths
$(2, 4, 6, 8, 10)$ used in the feedforward networks. In both cases we learn the hyperparameters using Algorithm \ref{alg:tuned-ucb} from $N=200$ training tasks with relatively small time horizons $T_o=20$. The learned hyperparameter values over the grid are $5.0, 0.1$ and $0.5$ respectively for the synthetic 2-arm data, CIFAR-10 and CIFAR-100. For the test run, we run the algorithms for five sampled test tasks with time horizon $T=10^6$ and report the mean and standard deviation of the regret for each time step in Figure \ref{fig:corral}. For the synthetic task, we generate the rewards according the Bernoulli distribution for the test tasks sampled from the problem distribution. For CIFAR we sample the task from the uniform task distribution and (since running neural network training is computationally intensive) we generate the first 20 runs as usual and simulate rewards for the remaining runs for each arm by using a Gaussian distribution with the same mean and standard deviation as the 20 time steps.

{\bf Rationale for the synthetic dataset in Figure  \ref{fig:corral}.} The goal for our synthetic data is to present a simple yet challenging setting for hyperparameter learning. Since corralling algorithms usually get worse with increased number of arms (as does our offline training complexity), two-arms is the simplest setting for corralling. By keeping a small gap in the arm rewards, typical instance-dependent regret bounds are weak for this problem and a large amount of exploration (relatively high $\alpha$ value) is needed to learn the best arm. On a longer time-horizon of $T=2\times 10^6$, the advantage of our approach becomes clearer even in this challenging case. Contrast this with Figure \ref{fig:corral2}, where the arm rewards are 0.5 and $\cN(0.7,0.01)$ and our approach is significantly better even on a much shorter timescale.

{\bf Uniform and Gaussian arm rewards} In addition to Bernoulli arm reward distributions, we also study the case of two-arm bandits with Uniform and Gaussian arm reward distributions and compare our approach with the corralling baselines (Figure \ref{fig:2arms-ug}). For the uniform distribution case, we set arm rewards as $U[0.5,0.6]$ and $U[0.55,0.65+\epsilon]$ and for the Gaussian distribution case, we set arm rewards as $\cN(0.5,0.1)$ and $\cN(0.65,0.1+\epsilon)$. Here $U[a,b]$ denotes the uniform distribution over the interval $[a,b]$ and $\epsilon\sim\cN(0,0.01)$. The rest of the setup is same as in the Bernoulli case above. Our algorithm outperforms both corralling algorithms in either case, with the difference being evident even on short timescales like $T=10000$.

\begin{figure*}[t]
\centering
\subfloat[Uniform arm rewards]{\includegraphics[width=0.392\textwidth]{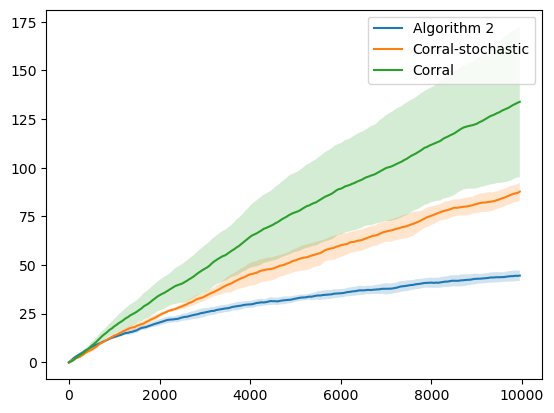}}
         \subfloat[Gaussian arm rewards]{\includegraphics[width=0.4\textwidth]{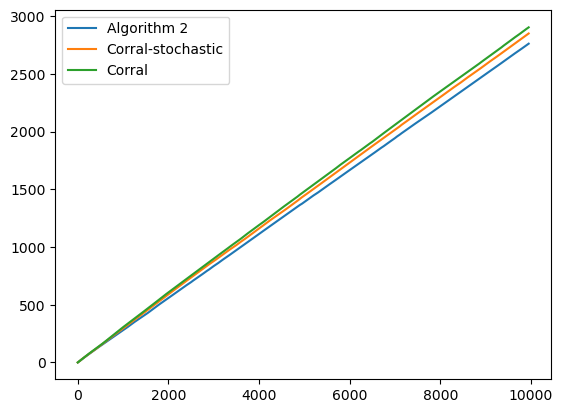}}

    \caption{Comparison of Algorithm \ref{alg:tuned-ucb} with corralling baselines on 2-arm bandits synthetic datasets with uniform and Gaussian arm reward distributions.}
    \label{fig:2arms-ug}
\end{figure*}

\section{Additional Experiments}\label{sec:experiments}


We empirically evaluate our theoretical results on synthetic as well as real benchmark data. To study the effect of the exploration parameter $\alpha$ in Algorithm \ref{alg:ucb}, we consider the following tasks

\textbf{Synthetic tasks.} We consider the following three families of arm reward distributions
\begin{itemize}[noitemsep]
    \item {\it Bernoulli}: Arm 1 draws a reward of 0 or 1 with probability 0.5 each in all tasks. Arm 2 draws a reward of value 1 with probability $p$ sampled from $\cN(0.5,\sigma_b^2=0.2)$ and 0 otherwise. 
    \item {\it Uniform}: Arm 1 draws a reward according to $U[2,6]$ and arm 2 according to $U[4.1-\sigma_u,4.1+\sigma_u]$ with $\sigma_u$ drawn from $\cN(1.5,0.5)$ for each task. Here $U[a,b]$ is the uniform distribution over $[a,b]$, and $\cN(\mu,\sigma^2)$ denotes the standard Gaussian distribution.
    \item {\it Gaussian}: Arm 1 draws a reward from $\cN(4,1)$, and arm 2 from $\cN(4.1,\sigma_g^2)$ with $\sigma_g^2\sim U[0.5,1.5]$.
\end{itemize}

\textbf{Hyperparameter tuning for Deep Learning.} We also consider the task of tuning the learning rate for training neural networks on  image classification tasks. The arms consist of 11 different learning rates $(0.001, 0.002, 0.004, 0.006, 0.008, 0.01, 0.05, 0.1, 0.2, 0.4, 0.8)$ and the arm reward is given by the classification accuracy of feedforward neural networks trained via SGD (stochastic gradient descent) with batch size 64 for 20 epochs. We present our results for CIFAR-10 and CIFAR-100~\cite{Krizhevsky09learning} benchmark image classification datasets. The task distribution
is defined by a uniform distribution over the label noise
proportions $(0.0, 0.1, 0.2, 0.3)$, and the network depth
$(2, 4, 6, 8, 10)$. All our experiments on CIFAR are run on $1$ Nvidia A100 GPU.

\textbf{Regret as a function of $\alpha$.} For both the synthetic and cifar experiments, we sample $N$ training tasks from the task distributions described as above, and compute the regret of UCB$(\alpha)$ for $\alpha\in[0,1]$ over $T_o=100$ rounds for synthetic experiments (20 rounds for CIFAR). For both these experiments, to estimate the dependence of expected regret on the exploration parameter, we take average over $N=10000$ training tasks (Figure \ref{fig:two-arm-alpha}). We observe that the regret is minimized for data-dependent $\alpha$. Further, the best parameter is often $\alpha<1$ for which previously known theoretical techniques do not provide a regret analysis (the techniques work for $\alpha\ge 1$ and suggest setting $\alpha=1$ to minimize the regret upper bound, which would be suboptimal).

\textbf{Number of discontinuities.} We also compute the average number of discontinuities in the dual loss function to estimate $Q_\cD$ for the different distributions. The distribution-specific upper bounds are typically much smaller than the worst case upper bounds that we obtain for arbitrary distributions. For example, for the Bernoulli setting above, $Q_\cD\approx 30\ll T^n$, implying that the sample complexity of learning the best hyperparameter is much smaller.

\begin{figure*}[t]
\centering
\subfloat[Two-arm bandits, simulated data]{\includegraphics[width=0.47\textwidth]{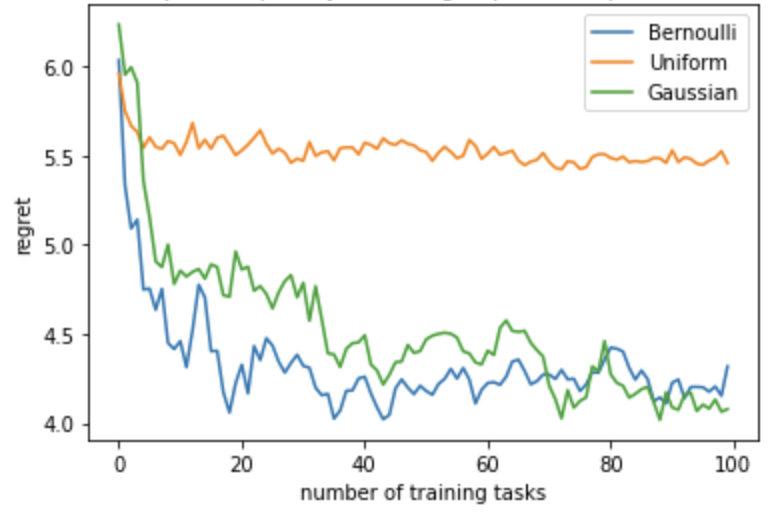}}
         \subfloat[Multi-arm bandits, CIFAR data]{\includegraphics[width=0.48\textwidth]{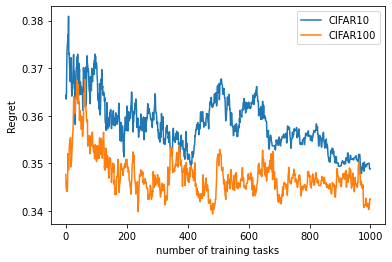}}

    \caption{Variation of  regret on test (online) tasks with number of training tasks $N$ for tuning \textsc{UCB}$(\alpha)$.}
    \label{fig:sc-alpha}
\end{figure*}

\begin{figure*}[b]
    \centering
         \subfloat[$\sin x +\cos y$]{\includegraphics[width=0.33\textwidth]{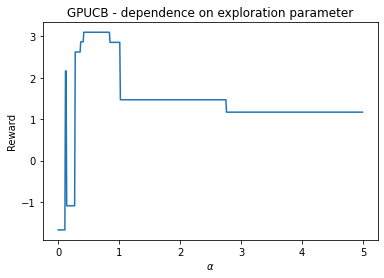}}
         \subfloat[$(x-1)(x+0.5)+(y-1.5)(y+1)$]{\includegraphics[width=0.33\textwidth]{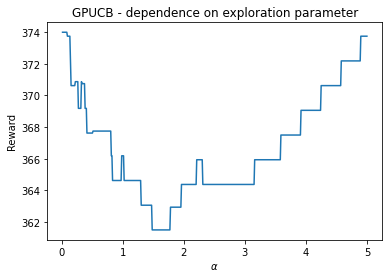}}
         \subfloat[$x+y^2$]{\includegraphics[width=0.33\textwidth]{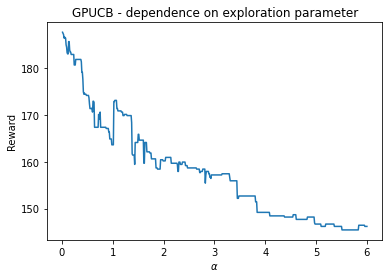}}
  \caption{Variation of total reward with $\sigma^2$ parameter in GP-UCB on a $24\times 24$ two-dimensional grid, for $T=20$. We observe that $Q_\cD\ll n$ in each case.}
\label{fig-gp}
\end{figure*}

\textbf{Quality of ERM Solution.} To estimate the generalization error of ERM solution, we consider $N_t=10$ test tasks with time horizon $T=100$ for synthetic (20  for CIFAR) drawn from the same distribution and use the learned parameter from $N$ training tasks. We use Algorithm \ref{alg:tuned-ucb} to learn the exploration parameter $\hat{\alpha}$ from the training tasks, which is then evaluated by computing the average regret of \textsc{UCB}$(\hat{\alpha})$ over the test tasks. As $N$ is varied, the generalization error of the learned parameter roughly decays quadratically with the number of training tasks (Figure \ref{fig:sc-alpha}). 

Note that we do not compare with corralling, and variance-aware algorithms because they work in a fundamentally different setting compared to us (our algorithms would out-perform them because they are not transferring knowledge). Consequently, this would be unfair to these techniques and wouldn't lead to any meaningful insights. Furthermore, $T_o$ (number of online rounds) is quite small in practical use cases we are interested in (such as HP Tuning in deep learning). For example, in our experiments on CIFAR, $T$ is set to 20. For such a small $T$, neither corralling nor variance estimation algorithms will be able to do much, because they simply don't have enough samples to infer the best parameter.

\textbf{GP-UCB.} We perform similar experiments for tuning the noise parameter in GP-UCB, for different underlying functions $f$. We consider the optimization of two-dimensional functions over a grid of $n=576$ points. Our experiments indicate the significance of learning data-driven values of $\sigma^2$, and also indicate $Q_\cD$ is much smaller than worst-case bounds (Figure \ref{fig-gp}). For $f(x,y)=\sin x + \cos y$, we observe $Q_\cD\approx 10$, implying a small inter and intra task complexity for learning the parameter.



\section{Empirical Size of $Q_\cD$}

We provide additional experiments studying the size of $Q_\cD$ which plays an important role in our sample complexity. For the multi-armed bandits (MAB) problem, 
we consider  tasks defined by the following three families of arm reward distributions.

\begin{itemize}[noitemsep]
    \item {\it Bernoulli}, $B(\sigma)$: Arm 1 draws a reward of 0 or 1 with probability 0.5 each in all tasks. Arm 2 draws a reward of value 1 with probability $p$ sampled from $\cN(0.5,\sigma^2)$ and 0 otherwise. As $\sigma$ is varied, we get different families. 
    \item {\it Uniform}, $U(\sigma)$: Arm 1 draws a reward according to $U[2,6]$ and arm 2 according to $U[4.1-\sigma_u,4.1+\sigma_u]$ with $\sigma_u$ drawn from $\cN(1.5,\sigma^2)$ for each task. 
    \item {\it Gaussian}, $G(\sigma)$: Arm 1 draws a reward from $\cN(4,1)$, and arm 2 from $\cN(4.1,\sigma^2)$.
\end{itemize}

In Table \ref{tab:qd}, we report the estimated values of $Q_\cD$ for tuning the exploration parameter in UCB for the above families of MAB problems. The typical number of pieces is much smaller than the theoretical upper bound of $n^T=2^{100}\approx 1.26\mathrm{E}30$. To compute reliable estimates, we average over 10000 runs for each task and report 95\% confidence intervals.

We observe similar small size of $Q_\cD$ in our experiments with the noise parameter in GP-UCB. For example, for $f(x,y)=\sin x + \cos y$, we observe $Q_\cD\approx 10$. This implies small inter as well as intra task complexity in this case, as $Q_\cD<n=576$.

\begin{table*}[t]
\centering

\begin{tabular}{cccc}
\toprule
$\sigma$\textbackslash Distribution &
Bernoulli, $B(\sigma)$ &
Uniform, $U(\sigma)$ &
Gaussian, $G(\sigma)$\\
\hline
$\sigma=0.1$ &
$28.26\pm1.05$ &
 $20.03\pm0.50$ &
$32.23\pm1.20$ \\
$\sigma=0.2$  &
$31.77\pm1.24$ &
 $20.53\pm0.52$ &
$28.70\pm1.01$ \\
$\sigma=0.3$  &
 $35.93\pm1.40$&
$19.79\pm0.50$  &
 $25.30\pm 0.85$ \\
$\sigma=0.5$ &
$40.84\pm1.57$ &
 $19.63\pm 0.53$ &
$22.48\pm0.68$ \\
\bottomrule
\end{tabular}
\caption{Estimates of $Q_\cD$ for various distributions corresponding to UCB parameterized by exploration parameter $\alpha\in[0,1]$ for time horizon $T=100$.}
\label{tab:qd}
\end{table*}



\end{document}